\documentclass[a4paper,fleqn]{cas-dc}
\usepackage[authoryear,longnamesfirst,round,sort&compress]{natbib}
\usepackage{hyperref}
\usepackage[dvipsnames]{xcolor}
\hypersetup{colorlinks=true,linkcolor=red,filecolor=green,      urlcolor=blue,citecolor=Cerulean}
\usepackage{amssymb}
\usepackage{lipsum}
\usepackage{mathtools}
\usepackage{cuted}
\usepackage{multicol}
\usepackage{amsmath}      
\usepackage{array}    
\usepackage{framed} 
\usepackage[titletoc]{appendix}
\usepackage{amsthm}
\usepackage[figuresright]{rotating}
\usepackage{float}
\usepackage{soul}
\usepackage{algorithmicx}
\usepackage{algpseudocode}
\usepackage{geometry}
\usepackage{fancyhdr}
\usepackage{flushend}
\graphicspath{{figures/}}
\usepackage{multirow} 
\usepackage{booktabs} 
\usepackage{caption}
\usepackage{epstopdf}

\usepackage{graphics}
\usepackage{epsfig}
\usepackage[ruled,linesnumbered]{algorithm2e}
\newtheorem{theorem}{Theorem}

\newtheorem{corollary}{Corollary}
\newtheorem{assumption}{Assumption}
\newtheorem{remark}{Remark}
\newtheorem{definition}{Definition}
\UseRawInputEncoding

\def\tsc#1{\csdef{#1}{\textsc{\lowercase{#1}}\xspace}}
\tsc{WGM}
\tsc{QE}
\tsc{EP}
\tsc{PMS}
\tsc{BEC}
\tsc{DE}

\begin{document}
\let\WriteBookmarks\relax
\def\floatpagepagefraction{1}
\def\textpagefraction{.001}



\title [mode = title]{Nonlinear sparse variational Bayesian learning based model predictive control with application to PEMFC temperature control}                      



%
\author[1]{Qi Zhang}[type=editor,
     style=chinese,
 orcid=0000-0000-0000-0000,]



\ead{zhang.qi@zju.edu.cn}




\author[1]{Lei Wang}[style=chinese]

\author[1]{Weihua Xu}[style=chinese]

\author[1]{Hongye Su}[style=chinese]

\author[1]{Lei Xie}[style=chinese]
\cormark[1]
\ead{leix@iipc.zju.edu.cn}
\affiliation[1]{organization={College of Control Science and Engineering},
     addressline={Zhejiang University}, 
    city={Hangzhou},
    postcode={310027}, 
    country={China}}

\cortext[cor1]{Corresponding author}



\begin{abstract}
The accuracy of the underlying model predictions is crucial for the success of model predictive control (MPC) applications.
If the model is unable to accurately analyze the dynamics of the controlled system, the performance and stability guarantees provided by MPC may not be achieved.
Learning-based MPC can learn models from data, improving the applicability and reliability of MPC.
This study develops a nonlinear sparse variational Bayesian learning based MPC (NSVB-MPC) for nonlinear systems, where the model is learned by the developed NSVB method.
Variational inference is used by NSVB-MPC to assess the predictive accuracy and make the necessary corrections to quantify system uncertainty.
The suggested approach ensures input-to-state (ISS) and the feasibility of recursive constraints in accordance with the concept of an invariant terminal region.
Finally, a PEMFC temperature control model experiment confirms the effectiveness of the NSVB-MPC method.
\end{abstract}




\maketitle

\section{Introduction}

Model Predictive Control (MPC) is an advanced control strategy that utilizes mathematical models to predict the future behavior of a system and compute optimal control inputs by repeatedly solving a finite-range optimization problem  \citep{rawlings2017model,cannonEfficientNonlinearModel2004,qinSurveyIndustrialModel2003}. 
The performance of MPC heavily relies on the accuracy of the model predictions, which are usually derived mathematically using a simplified physical modeling approach \citep{maciejowski2002predictive,gruneNonlinearModelPredictive2017a}. 
However, accurate modeling of complex systems based on first principles requires an in-depth understanding of the process and is often difficult or costly to apply in practice. 
In the absence of plant dynamics models, machine learning-based approaches can be used to learn these models from observed data \citep{breschiDatadrivenPredictiveControl2023,terziLearningbasedPredictiveControl2019a}. 
Learning-based MPC depends on two phases: model learning and controller design.
The aim of this approach is to improve model dynamics by incorporating data-driven components that reduce uncertainty. 
In recent years, machine learning techniques, such as neural networks \citep{bemporadTrainingRecurrentNeural2023a,chenNonlinearSystemIdentification1990} and Gaussian process (GP) regression \citep{kocijan2016modelling,seeger2004gaussian}, have demonstrated the ability to approximate nonlinear systems and have shown great potential for model learning.

However, measurement data from large-scale plants can become complex. Valuable information is often mixed with extraneous noise perturbations, and machine learning algorithms can suffer from overfitting and lack of interpretability  \citep{bruntonDiscoveringGoverningEquations2016a,kaiserSparseIdentificationNonlinear2018a}. 
Additionally, many systems must be identified and controlled with limited data, neural networks and deep learning require large amounts of training data  \citep{nubertSafeFastTracking2020}. 
GP is not suitable for real-time control due to the computational complexity of nonlinear optimisation \citep{maiwormOnlineLearningbasedModel2021}.
To address these challenges simultaneously, sparse Bayesian learning (SBL) methods provide a principled approach \citep{zhangDynamicFaultDetection2024,tipping2001sparse,wipfSparseBayesianLearning2004}. 
During the model learning phase, Bayesian inference utilizes prior knowledge and posterior probabilities to select the optimal model, providing a measure of uncertainty in parameter estimation. 
To further quantify parameter uncertainty, \citep{jacobsSparseBayesianNonlinear2018} developed a nonlinear autoregressive exogenous (NARX) model identification method based on SBL identification \citep{panSparseBayesianApproach2016a} with variational inference \citep{bealVariationalBayesianLearning2006,zhangVariationalBayesianState2023}.

Although learning-based MPC adjusts and compensates for uncertainty during the learning phase, the primary challenge is to guarantee the stability and robustness of closed-loop systems \citep{limonInputStateStability2006}. Robustness implementations can be classified into three main categories. 
The first category involves minimizing the nominal performance index by imposing strict state and terminal constraints  \citep{wuStatisticalMachinelearningBased2022,munozdelapenaLyapunovBasedModelPredictive2008}. 
The second category involves solving a Min-max optimization problem \citep{scokaertMinmaxFeedbackModel1998, yanRobustModelPredictive2014}. 
The third category involves computing the disturbance invariant sets to ensure that the closed-loop trajectory evolves within an invariant tube \citep{mayneTubebasedRobustNonlinear2011,aswaniProvablySafeRobust2013a,nguyenHighprobabilityStableGaussian2022a}.
Robust optimization involves constructing invariant sets that include possible realizations of uncertain parameters. 
The computation procedure for terminal regions offline is complex and cumbersome \citep{jiangInputtostateStabilityDiscretetime2001}. Therefore, the MPC scheme without terminal constraints is much simpler for design and implementation. This has led to the study of many approaches to predictive controllers without terminal constraints \citep{mullerEconomicModelPredictive2016}.
\citep{limonStabilityConstrainedMPC2006} suggest weighting the terminal costs to enlarge the domain of attraction of the MPC controller without terminal constraints, achieving the same domain of attraction as the MPC controller with terminal constraints.
According to the concept of invariant terminal regions  \citep{limonInputtoStateStabilityUnifying2009}, the control stage inherits the inherent robust stability of the nominal MPC under constraints if the system under the predictive controller proves to be input-to-state stability (ISS) with respect to predictive error.

This study develops a nonlinear sparse variational Bayesian learning-based MPC (NSVB-MPC) to improve data-driven control of nonlinear systems. 
In the learning phase, model selection is performed through sparsity-induced priors. 
By employing the method of automatic relevance determination (ARD), NSVB-MPC can iteratively prune redundant terms from the model.
In addition, on the basis of probability statistics, variational Bayes can not only model the uncertainty of model structure, noisy data, and system parameters, but also characterize the uncertainty of parameters as a probability distribution, which can improve the  estimation accuracy of model parameters. 
As a result, during NSVB-MPC learning, the uncertainty of the posterior probability representation will continue to decrease as the learned dataset increases  \citep{bishop2006pattern}.
The control phase addresses the robust stability of the nominal MPC under constraints. 
It is not necessary to increase the computational complexity associated with predictive control optimization when the terminal state constraints are introduced for stability reasons.
To verify the effectiveness of the NSVB-MPC method, we implement the PEMFC temperature control systems by simulations.

The primary contributions of this study can be outlined as follows:
\begin{itemize}
	\itemsep=0pt
	\item
	Nonlinear sparse variational Bayesian learning algorithms are developed for constructing polynomial input-output models.
	\item 
	The proposed approach constructs a sparse prior which naturally reduces the impact of noise and quantifies uncertainty through variational inference.
	\item
	A  robust predictive controller without terminal constraints is built based on a predictive model that avoids robust invariant region estimation and provides stability guarantees.
\end{itemize}

The subsequent sections of this paper are structured as follows: Section~\ref{sec:related} delves into the problem description, Section~\ref{sec:VB} covers the nonlinear variational Bayesian sparse learning method, Section~\ref{sec:mpc} discusses predictive control with stability, Section~\ref{sec:experiments} presents simulation results, and Section~\ref{sec:conclusions} provides concluding remarks.

\section{Problem Formulation} \label{sec:related}

The objective of this work is to effectively control an unknown discrete-time nonlinear system.
Suppose that the only information available to the plant is the collected data $\mathcal{D}$, which includes specific control inputs $ \boldsymbol{u} \in \mathbb{R}^{n_u}$ and measurement outputs $\boldsymbol{y}  \in \mathbb{R}^{n_y} $.
The system needs to be satisfied with hard constraints on the input and output sequences.
\begin{align}
	\boldsymbol{u}  \in \mathcal{U} ,\quad  \boldsymbol{y} \in \mathcal{Y}
\end{align}
where $\mathcal{U} \in \mathbb{R}^{n_u} $ and $\mathcal{Y}  \in \mathbb{R}^{n_y}$ are closed sets whose interior contains the origin. 
To maintain generality, it is assumed that the system equilibrium is located at the origin.

\begin{assumption}  \label{ass1}
	Given the limited availability of inputs and outputs, it is assumed that such a nonlinear dynamic system can be characterized through the utilization of the NARX model \citep{sjobergNonlinearBlackboxModeling1995}.
	\label{cor-1}
	The nonlinear dynamic system can be written as follows:
	\begin{align} \label{2}
		\begin{split}
			{y}_{k+1}=&{f}(\boldsymbol{x}_k,{u}_k)+{\xi}_{k} \\
		\end{split}
	\end{align}
	where ${f}(\cdot)$ is the unknown nonlinear function.  The horizons $n_a$ and $n_b$ denote the model order. $\boldsymbol{x}_k$ is the state vector. $\boldsymbol{x}_k =[ {y}_k, \cdots , {y}_{k-n_a}, {u}_{k-1}, \cdots {u}_{k-n_b}] \in \mathcal{X}:=\mathcal{Y}^{n_a+1} \times \mathcal{U}^{n_b} \subseteq 
	\mathbb{R}^{n_x} $ with $n_x=n_a+n_b+1$ . 
	${\xi}_{k}$  is a zero mean Gaussian distributed white noise with precision parameter $\beta$.
	For the sake of simplicity in notation, the inputs of the function ${f}(\cdot)$ are consolidated into a unified vector
	$\boldsymbol{z}_k:=(\boldsymbol{x}_k,{u}_k)  \in \mathbb{R}^{ n_z}$ with the order $n_z=n_y+n_u$, which is referred to as regressor.
\end{assumption}

\begin{assumption} 
	The assumption is made that the reference point is an equilibrium that complies with the constraints, such that $\boldsymbol{y}_r \in \mathcal{Y}$ and $\boldsymbol{u}_r \in \mathcal{U}$.
	\label{cor-2}
\end{assumption}

To calculate a control action using the input-output training data $\mathcal{D}$, the current output measurement $\boldsymbol{y}_k$ and reference $\boldsymbol{y}_r$, an output feedback control law is designed as
\begin{align}
	\boldsymbol{u}_k=\kappa_{f}(\boldsymbol{y}_k,\boldsymbol{y}_r;\mathcal{D}) .
\end{align}

\begin{definition}
	The nonlinear function ${f}(\cdot)$ described in \eqref{2} can be written as a generic set of observation vectors $\boldsymbol{\Phi}$, also known as a dictionary, which allows the learning task to be simplified and the model complexity to be reduced.
	${f}(\cdot)$ can be broken down into a sum of $M$ basis functions, each weighted accordingly \citep{jacobsSparseBayesianNonlinear2018,panSparseBayesianApproach2016a}
	\begin{align} \label{4}
		{f}\left( \boldsymbol{z}_k \right) =\sum \limits _{i=1}^M  \boldsymbol{\phi} _m (\boldsymbol{z}_{k}) \omega _m
		= {\Phi}_{k} \boldsymbol{\omega }
	\end{align}
	where  
	${\Phi}_{k}=\left[ {\phi} _1 ( \boldsymbol{z}_{k} ) , {\phi} _2 ( \boldsymbol{z}_{k} ) , \cdots , {\phi} _M (\boldsymbol{z}_{k} )\right] \in \mathbb{R}^{ M} $ is the $k$th row of the matrix $\boldsymbol{\Phi} $ , 
	$\boldsymbol{\Phi}=[\Phi^{\top}_{1}  , \Phi^{\top}_{2} , \cdots ,\Phi^{\top}_{N}      ]^{\top} 
	\in \mathbb{R}^{N \times M}$,
	$\boldsymbol{\omega}=\left[ {\omega} _1 , {\omega} _2   , \cdots , {\omega} _M \right]^{\top} \in \mathbb{R}^{M}$ is the unknown coefficient vector.
	\label{thm-1}
\end{definition}

\begin{remark}
	For the nonlinear system model \eqref{2}, the uncertainty comes mainly from the effect of noise ${\xi}_{k}$. The predicted $y_k$ for a given input $\boldsymbol{u}_k$ is uncertain due to the disturbance caused by noise in the observed data.
	Probability theory offers a sensible way for quantifying and calculating uncertainty. It aims to model the predictive distribution $p(y_k|\boldsymbol{z})$ as it expresses uncertainty regarding the output $\boldsymbol{y}$ for each value of regressor $\boldsymbol{z}$ \citep{bishop2006pattern}.
\end{remark}

\section{Nonlinear Sparse Variational Bayesian Learning} \label{sec:VB}

In Bayesian inference, all unknowns are treated as stochastic variables characterized by specific probability distributions,
under the noise ${\xi} \sim \mathcal{N}(0,\beta^{-1})$,
the likelihood function for the NARX model given in \eqref{2} with $\boldsymbol{y}$ is 
\begin{align} \label{5}
	\begin{split}			
		&p(\boldsymbol{y}|\boldsymbol{\Phi},\boldsymbol{\omega },\beta )	=\prod\limits_{\text{k}=1}^{N}{\mathcal{N}\left( {{{y}}_{k}}|{{\boldsymbol{\omega }}^{\top }}{\Phi}_k ,{{\beta }^{-1}} \right)} \\ 
		& ={{\left( \frac{\beta }{2\pi } \right)}^{\frac{N}{2}}}\text{exp}\left\{ -\frac{\beta }{2}\sum\limits_{k=1}^{N}{{{\left[ {{ {y}}_{k}}-{{\boldsymbol{\omega }}^{\top }}{\Phi}_k \right]}^{2}}} \right\} .
	\end{split}
\end{align}

The $\boldsymbol{\omega}$ is given as a Gaussian prior
\begin{align} \label{6}
	p(\boldsymbol{\omega}|\boldsymbol{\alpha}) =\prod\limits_{\text{i}=1}^{M}{\mathcal{N}\left( {\omega_m}|0,{{\alpha}_m^{-1}} \right)}
\end{align}
where  $\boldsymbol{\alpha}=(\alpha_1,\alpha_2,\cdots,\alpha_M)=diag(\boldsymbol{A})$ denotes the vector of hyperparameters,
and the prior distribution of the hyperparameter $\alpha$ is introduced, i.e., the Gaussian prior accuracy of the weights
\begin{align} \label{7}
	p(\alpha_m) =\text{Gam}(\alpha_m|a_0,b_0)
\end{align}
where $\text{Gam}$ denotes the gamma distribution, $a_0$ and $b_0$ denote hyperparameters.

Similarly, a prior over the inverse noise variance $\beta$ in \eqref{5} is introduced
\begin{align}  \label{8}
	p(\beta)=\text{Gam}(\beta|c_0,d_0).
\end{align}
where $c_0$ and $d_0$ denote hyperparameters.

Referring to the likelihood function in \eqref{5}, the parameters of the model that need to be learned are denoted as $\boldsymbol{\Theta}=( \boldsymbol{\omega}, \boldsymbol{\alpha}, \beta)$ given the $\mathcal{D}$.
	For the measurement output $\boldsymbol{y}$ and the parameters $\boldsymbol{\Theta}$, with the Bayesian rule, the posterior distribution $p(\boldsymbol{\Theta}|\boldsymbol{y})$ takes the following form
	\begin{align} \label{9}
		p(\boldsymbol{\Theta}|\boldsymbol{y})=\frac{p(\boldsymbol{\Theta},\boldsymbol{y})}{p(\boldsymbol{y})} = \frac{p(\boldsymbol{y}|\boldsymbol{\Phi},\boldsymbol{\Theta})p(\boldsymbol{\Theta})}{p(\boldsymbol{y})} 
	\end{align}
	where $p(\boldsymbol{\Theta},\boldsymbol{y})$ is the joint distribution.
	$p(\boldsymbol{y}|\boldsymbol{\Phi},\boldsymbol{\Theta})$ indicates the likelihood function, which is the probability that $\boldsymbol{y}$ is observed given $\boldsymbol{\Theta}$. $p(\boldsymbol{\Theta})$ is the prior distribution of the parameters before observing the data. $p(\boldsymbol{y})$ is the marginal distribution, which can be obtained through integrating over $\boldsymbol{\Theta}$
	\begin{align} 
		p(\boldsymbol{y})=\int_{\boldsymbol{\Theta}}p(\boldsymbol{\Theta},\boldsymbol{y})d\boldsymbol{\Theta}  .
	\end{align}

	In \eqref{9}, the model parameters $\boldsymbol{\Theta}$ given the $\mathcal{D}$ now incorporate the hyperparameter $\boldsymbol{\alpha}$, which makes obtaining an analytical solution for the marginal distribution $p(\boldsymbol{y})$ challenging. 
	Hence, a fundamental challenge in Bayesian methods lies in replacing the computationally intractable true posterior distribution $p(\boldsymbol{\Theta}|\boldsymbol{y})$ with a simpler, analytically tractable distribution.
	Commonly used methods include random sampling \citep{gilks1995markov}, but this method leads to higher computational costs.
	Therefore, this work employs variational inference to approximate the posterior distribution $p(\boldsymbol{\Theta}|\boldsymbol{y})$ using a series of closed-form update equations. 
	The goal of variational inference is to find an approximate distribution $Q(\boldsymbol{\Theta})$ that can replace the true posterior distribution $p(\boldsymbol{\Theta}|\boldsymbol{y})$. 
	To achieve highly accurate approximations of the true posterior distribution $p(\boldsymbol{\Theta}|\boldsymbol{y})$, it is essential to establish Evidence Lower Bound (ELBO) $\mathcal{L}(Q(\boldsymbol{\Theta}))$ for inference, ensuring the optimality of $\mathcal{L}(Q(\boldsymbol{\Theta}))$.

	\begin{theorem} \label{th1}
		For the approximate distribution $Q(\boldsymbol{\Theta})$ and the log-likelihood function of the marginal distribution $\ln p(\boldsymbol{y})$, 
		if there is ELBO $\mathcal{L}(Q(\boldsymbol{\Theta}))$ for variational inference, then the relationship between log-likelihood function $\ln p(\boldsymbol{y})$ satisfies
		\begin{align}
			\ln p(\boldsymbol{y})
			\geq \mathcal{L}(Q(\boldsymbol{\Theta})).
		\end{align}
	\end{theorem}

\begin{proof}
		Introducing the approximate distribution $ Q(\boldsymbol{\Theta})$, the log-likelihood function $\ln p(\boldsymbol{y})$ takes the following form
		\begin{align} \label{12}
			\begin{split}
				\ln p(\boldsymbol{y})
				&=\ln \int_{\boldsymbol{\Theta}} Q(\boldsymbol{\Theta})  \left\lbrace \frac{p(\boldsymbol{\Theta},\boldsymbol{y})}{Q(\boldsymbol{\Theta})} \right\rbrace d \boldsymbol{\Theta} \\
				&= \int_{\boldsymbol{\Theta}}Q(\boldsymbol{\Theta}) \ln \left\lbrace \frac{p(\boldsymbol{\Theta},\boldsymbol{y})}{Q(\boldsymbol{\Theta})} \right\rbrace d \boldsymbol{\Theta}
				\\
				&-\int_{\boldsymbol{\Theta}}Q(\boldsymbol{\Theta}) \ln \left\lbrace \frac{p(\boldsymbol{\Theta}|\boldsymbol{y})}{Q(\boldsymbol{\Theta})} \right\rbrace d \boldsymbol{\Theta}\\
				&=\mathcal{L}(Q(\boldsymbol{\Theta}))+\mathbb{KL}\left( Q(\boldsymbol{\Theta})||p(\boldsymbol{\Theta}|\boldsymbol{y})\right)
			\end{split}
		\end{align}
		where $\int_{\boldsymbol{\Theta}}Q(\boldsymbol{\Theta}) \ln \left\lbrace \frac{p(\boldsymbol{\Theta},\boldsymbol{y})}{Q(\boldsymbol{\Theta})} \right\rbrace d \boldsymbol{\Theta}$ denotes the ELBO $\mathcal{L}(Q(\boldsymbol{\Theta}))$, $-\int_{\boldsymbol{\Theta}}Q(\boldsymbol{\Theta}) \ln \left\lbrace \frac{p(\boldsymbol{\Theta}|\boldsymbol{y})}{Q(\boldsymbol{\Theta})} \right\rbrace d \boldsymbol{\Theta}$ is the Kullback-Leibler (KL) divergence $\mathbb{KL}\left( Q(\boldsymbol{\Theta})||p(\boldsymbol{\Theta}|\boldsymbol{y})\right)$.
		$\mathbb{KL}(Q(\boldsymbol{\Theta})||p(\boldsymbol{\Theta}|\boldsymbol{y}))$ quantifies the KL divergence between
		$Q(\boldsymbol{\Theta})$ and the posterior distribution $p(\boldsymbol{\Theta}|\boldsymbol{y})$. Since KL divergence $\mathbb{KL}(Q(\boldsymbol{\Theta})||p(\boldsymbol{\Theta}|\boldsymbol{y}))$ has non-negativity,  we know $\ln p(\boldsymbol{y})
		\geq \mathcal{L}(Q(\boldsymbol{\Theta}))$.
	
\end{proof}

	Applying Theorem \ref{th1}, the log-likelihood function $\ln p(\boldsymbol{y})$ can be updated by $\mathcal{L}(Q(\boldsymbol{\Theta}))$.
	The objective of minimizing $\mathbb{KL}(Q(\boldsymbol{\Theta})||p(\boldsymbol{\Theta}|\boldsymbol{y}))$ is equivalent to maximizing $\mathcal{L}(Q(\boldsymbol{\Theta}))$. 
	By maximizing  $\mathcal{L}(Q(\boldsymbol{\Theta}))$, the optimal approximate distribution $Q(\boldsymbol{\Theta})$ for the marginal distribution $p(\boldsymbol{\Theta}|\boldsymbol{y})$ can be found. 
	When $Q(\boldsymbol{\Theta})$ equals $p(\boldsymbol{\Theta}|\boldsymbol{y})$, the KL divergence becomes negligible, resulting in an exact match between $Q(\boldsymbol{\Theta})$ and $p(\boldsymbol{\Theta}|\boldsymbol{y})$ \citep{kullbackInformationSufficiency1951}.

	However, the approximate distribution $Q(\boldsymbol{\Theta})$ encompasses all the model parameters, which makes the analytical solution still challenging to obtain. If the approximate distribution can be expressed as an independent superposition of each variable, then the integral of $Q(\boldsymbol{\Theta})$ can be decomposed into a superposition of multiple factor distribution $q_j(\Theta_j)$. This decomposition enables closed-form updating equations for each factor.

	\begin{assumption} \label{qd}
		It is assumed that the approximate distribution $Q(\boldsymbol{\Theta})$ can be decomposed into an independent superposition of each factor distribution $q_j(\Theta_j)$
		\begin{align}
			Q(\boldsymbol{\Theta}) =  \prod^{I} _ { i } q _ { i } ( \Theta _ { i } )
		\end{align}
	\end{assumption}

	With this decomposition, 
	the objective of maximizing $\mathcal{L}(Q(\boldsymbol{\Theta}))$ becomes more straightforward.
	It can be accomplished by optimizing each factor distribution individually with respect to all other distributions.

	\begin{theorem} \label{th2}
		Let Assumption \ref{qd} be satisfied, the optimal variational distribution $q^{\ast}_j(\Theta_j)$ for the factor distribution $q_j(\Theta_j)$ satisfies the following
		\begin{align} \label{12a}
			q^{\ast}_j(\Theta_j)=\frac{ \exp \mathbb{E}_{i\neq j}(\ln p(\boldsymbol{y},\boldsymbol{\Theta})) }{\int_{\Theta_j} \exp \left(  \mathbb{E}_{i\neq j}(\ln p(\boldsymbol{y},\boldsymbol{\Theta}))  \right) } .
		\end{align}
	\end{theorem}

\begin{proof}
		With the Assumption \ref{qd},
		$\mathcal{L}(Q(\boldsymbol{\Theta}))$ in \eqref{12} can be decomposed into
		\begin{align} \label{14}
			\begin{split}
				&\mathcal{L}(Q(\boldsymbol{\Theta}))
				=\int_{\Theta_j} \prod _ { i } q_i({\Theta_i}) \ln \left\{ \frac { p ( \boldsymbol{\Theta} , \boldsymbol{y} ) } { \prod _ { i } q_i({\Theta_i}) } \right\} d \boldsymbol{\Theta}_i  \\
				&=\int_{\Theta_j}q_j({\Theta_j})\int_{\Theta_i,i\neq j}\prod \limits_{\text{i}=1}q_i ({\Theta_i}) \ln p(\boldsymbol{y},\boldsymbol{\Theta})d\Theta_i d\Theta_j  \\&-\int_{\Theta_j}q_j({\Theta_j})\int_{\Theta_i,i\neq j}\prod \limits_{\text{i}=1}q_i ({\Theta_i}) 
				\ln q_j(\Theta_j)d\Theta_i d\Theta_j \\
				&+\text{const}  \\
			\end{split}
		\end{align}
		where $\text{const}$  is a constant term, independent of the parameter $\Theta_j$.

		Eliminating the term independent of the parameter $\Theta_j$, the second term in the right-hand side of \eqref{14} can be simplified as
		\begin{align} \label{15}
			\begin{split}
\int_{\Theta_j}q_j({\Theta_j})\int_{\Theta_i,i\neq j}\prod \limits_{\text{i}=1}q_i ({\Theta_i}) 
				\ln q_j(\Theta)d\Theta_i d\Theta_j
				\\
				= \int_{\Theta_j}q_j({\Theta_j}) \ln q_j(\Theta_j) d\Theta_j.
			\end{split}
		\end{align}

		Substituting \eqref{15} into \eqref{14}, $\mathcal{L}(Q(\boldsymbol{\Theta}))$ can be further simplified as
		\begin{align} \label{lq1}
			\begin{split}
				&\mathcal{L}(Q(\boldsymbol{\Theta}))
				= \text{const}-\int_{\Theta_j}q_j({\Theta_j}) \ln q_j(\Theta_j) d\Theta_j \\
				&+\int_{\Theta_j}q_j({\Theta_j})\int_{\Theta_i,i\neq j}\prod \limits_{\text{i}=1}q_i ({\Theta_i}) \ln p(\boldsymbol{y},\boldsymbol{\Theta})d\Theta_i d\Theta_j
			\end{split}
		\end{align}
		where $\int_{\Theta_i,i\neq j}\prod \limits_{\text{i}=1}q_i ({\Theta_i}) \ln p(\boldsymbol{y},\boldsymbol{\Theta})d\Theta_i$ is the expectation of the log-joint distribution such that 
		\begin{align} \label{16}
			\int_{\Theta_i,i\neq j}\prod \limits_{\text{i}=1}q_i ({\Theta_i}) \ln p(\boldsymbol{y},\boldsymbol{\Theta})d\Theta_i = \mathbb{E}_{i\neq j}(\ln p(\boldsymbol{y},\boldsymbol{\Theta}))
		\end{align}
		where $\mathbb{E}(\cdot)$ is the expectation, $\mathbb{E}_{i\neq j}$  denotes the expectation considering the distributions $q$ over all variables within the  $\boldsymbol{\Theta}$ for which $i \neq j$.

		Following \eqref{16}, $\mathcal{L}(Q(\boldsymbol{\Theta}))$ in \eqref{lq1} can be simplified as follows
		\begin{align} \label{17}
			\begin{split}
				\mathcal{L}(Q(\boldsymbol{\Theta}))&=\int_{\Theta_j}q_j(\Theta_j)\mathbb{E}_{i\neq j}(\ln p(\boldsymbol{y},\boldsymbol{\Theta}))d\Theta_j\\
				&-\int_{\Theta_j}q_j(\Theta_j) \ln q_j(\Theta_j)d\Theta_j + \text{const} 
			\end{split}
		\end{align}

		Referring to the KL divergence defined in \eqref{12}, it is known that the right-hand side of \eqref{17} is equivalent to a negative KL divergence between $q_j(\Theta_j)$ and $\mathbb{E}_{i\neq j}(\ln p(\boldsymbol{y},\boldsymbol{\Theta}))$, $\mathcal{L}(Q(\boldsymbol{\Theta}))$ in \eqref{17} can be rewritten as
		\begin{align}
			\begin{split}
				\mathcal{L}(Q(\boldsymbol{\Theta}))&=-\mathbb{KL}(q_j(\Theta_j)||\mathbb{E}_{i\neq j}(\ln p(\boldsymbol{y},\boldsymbol{\Theta}))+\text{const} 
			\end{split}
		\end{align}

		Thus maximizing $\mathcal{L}(Q(\boldsymbol{\Theta}))$ is equivalent to minimizing the KL divergence between $q_j(\Theta_j)$ and $\mathbb{E}_{i\neq j}(\ln p(\boldsymbol{y},\boldsymbol{\Theta}))$, and the minimum occurs when $q_j(\Theta_j)$ equals $\mathbb{E}_{i\neq j}(\ln p(\boldsymbol{y},\boldsymbol{\Theta}))$.
		Define the optimal variational distribution $q^{\ast}_j(\Theta_j)$ of the factor distribution $q_j(\Theta_j)$ to approximate the expectation $\mathbb{E}_{i\neq j}(\ln p(\boldsymbol{y},\boldsymbol{\Theta}))$. The logarithmic form of the optimal variational distribution  $q^{\ast}_j(\Theta_j)$ given by
		\begin{align} \label{18}
			\ln q^{\ast}_j(\Theta_j)=\mathbb{E}_{i\neq j}(\ln p(\boldsymbol{y},\boldsymbol{\Theta})) + \text{const}.
		\end{align}

		The additive constant in \eqref{18} is set by normalizing the optimal variational distribution $q^{\ast}_j(\Theta_j)$.
		Thus, taking the exponential of both sides of \eqref{18} and normalize,  $q^{\ast}_j(\Theta_j)$ given by
		\begin{align}
			q^{\ast}_j(\Theta_j)=\frac{ \exp \mathbb{E}_{i\neq j}(\ln p(\boldsymbol{y},\boldsymbol{\Theta})) }{\int_{\Theta_j} \exp \left(  \mathbb{E}_{i\neq j}(\ln p(\boldsymbol{y},\boldsymbol{\Theta}))  \right) } .
		\end{align}

\end{proof}

	With the Theorem \ref{th2}, performing the update of \eqref{12a} for each factor distribution $q_i(\Theta_i)$  completes one optimization step of $\mathcal{L}(Q(\boldsymbol{\Theta}))$ . These equations can be iteratively updated sequentially by using the statistics of $q^{\ast}_j(\Theta_j)$ from the previous iteration, resulting in an expectation-maximization (EM)-like algorithm.

	Review the NARX model in \eqref{5}. The parameter to be learned is $\boldsymbol{\Theta}=( \boldsymbol{\omega}, \boldsymbol{\alpha}, \beta)$. According to the Bayesian principle, the parameter is obtained by calculating the posterior distribution $p(\boldsymbol{\Theta}|\boldsymbol{y})$. 
	According to Theorem \ref{th1}, the computation of the posterior distribution $p(\boldsymbol{\Theta}|\boldsymbol{y})$ can be approximated by an approximate distribution $Q(\boldsymbol{\Theta})$. 
	Therefore, the learning of the model parameters is actually achieved by optimizing this approximate distribution $Q(\boldsymbol{\Theta})$. 
	Furthermore, according to Theorem \ref{th2}, the purpose of optimizing the approximate distribution $Q(\boldsymbol{\Theta})$ can be achieved by optimizing each factor distribution $q_j(\Theta_j)$, thus completing the learning of model parameters
	\begin{align}
		\begin{split}
	p(\boldsymbol{\Theta}|\boldsymbol{y})&= p(\boldsymbol{\omega},\boldsymbol{\alpha},\beta|\boldsymbol{y}) \simeq Q(\boldsymbol{\omega},\boldsymbol{\alpha},\beta) \\
			&= 	{{Q}_{\boldsymbol{\omega}}}\left( \boldsymbol{\omega} \right) 	{{Q}_{\boldsymbol{\alpha}}}\left( \boldsymbol{\alpha} \right) {Q}_{\beta}\left( \beta \right).
		\end{split}
	\end{align}

	Applying the conclusion of Theorem \ref{th2}, the optimal variational distribution $q^{\ast}_j(\Theta_j)$ will be computed by a series of closed-form update equations through \eqref{18}. 
	The optimal variational distribution for ${{Q}_{\boldsymbol{\omega}}}\left( \boldsymbol{\omega} \right)$, 	${{Q}_{\boldsymbol{\alpha}}}\left( \boldsymbol{\alpha} \right)$ and ${Q}_{\beta}\left( \beta \right)$ will be obtained by the following three corollaries, respectively.

\begin{corollary} \label{co1}
	Following the Theorem \ref{th2}, ${{Q}^{\ast}_{\boldsymbol{\omega}}}\left( \boldsymbol{\omega} \right)$ is a Gaussian distribution
	\begin{align} \label{26}
		{{Q}^{\ast}_{\boldsymbol{\omega}}}\left( \boldsymbol{\omega} \right)=\mathcal{N}(\boldsymbol{\omega}|\boldsymbol{\mu}_{\boldsymbol{\omega}},\boldsymbol{\Sigma}_{\boldsymbol{\omega}})
	\end{align}
	where $\boldsymbol{\Sigma}_{\boldsymbol{\omega}}= \left( \operatorname{diag} \mathbb{E}_{\boldsymbol{\alpha}}[  \alpha_m  ] + \mathbb{E}_{\beta}[\beta] \sum\limits_{k=1}^{N} {\Phi}^{\top}_k {\Phi}_k \right) ^{-1}$, $\boldsymbol{\mu}_{\boldsymbol{\omega}} = \mathbb{E}_{\beta}[\beta] \boldsymbol{\Sigma}_{\boldsymbol{\omega}}\sum\limits_{k=1}^{N} {\Phi}^{\top}_k\boldsymbol{y}$.
\end{corollary}

\begin{proof}

		According to Theorem \ref{th2},  the logarithmic form of the optimal variational distribution $ {{Q}^{\ast}_{\boldsymbol{\omega}}}\left( \boldsymbol{\omega} \right) $ can be given by applying \eqref{18}:
		\begin{align}  
			\label{20}	\ln {{Q}^{\ast}_{\boldsymbol{\omega}}}\left( \boldsymbol{\omega} \right)&={{\mathbb{E}}_{\boldsymbol{\alpha},\beta }}\left[ \ln p\left( \mathbf{y},\boldsymbol{\alpha}, \beta ,\boldsymbol{\omega} \right) \right]+\operatorname { c o n s t }.
		\end{align}

		Only those terms that have some functional dependence on $\boldsymbol{\omega}$ need to be retained, all other terms can be absorbed into the normalization constant.
		$\ln {{Q}^{\ast}_{\boldsymbol{\omega}}}\left( \boldsymbol{\omega} \right)$ can be further simplified as
		\begin{align} \label{25}
			\begin{split}
				&\ln {{Q}^{\ast}_{\boldsymbol{\omega}}}\left( \boldsymbol{\omega} \right)
				={{\mathbb{E}}_{\boldsymbol{\alpha},\beta }}\left[ \ln p\left( \mathbf{y}|\boldsymbol{\Phi},\boldsymbol{\omega},\boldsymbol{\alpha}, \beta  \right) \right] \\
				&+ {{\mathbb{E}}_{\boldsymbol{\alpha},\beta }}\left[ \ln p\left( \boldsymbol{\omega}|\boldsymbol{\alpha}  \right) \right] +\operatorname { c o n s t }\\
				&=\ln p\left( \mathbf{y}|\boldsymbol{\Phi},\boldsymbol{\omega},\boldsymbol{\alpha}, \beta  \right) + {{\mathbb{E}}_{\boldsymbol{\alpha},\beta }}\left[ \ln p\left( \boldsymbol{\omega}|\boldsymbol{\alpha}  \right) \right] +\operatorname { c o n s t }\\
			\end{split}
		\end{align}
		where  $ p\left( \mathbf{y}|\boldsymbol{\Phi},\boldsymbol{\omega},\boldsymbol{\alpha}, \beta  \right)$ and $p\left( \boldsymbol{\omega}|\boldsymbol{\alpha}  \right)$ are calculated through \eqref{5} and \eqref{6}, respectively. 
		Expanding the logarithmic form of $ p\left( \mathbf{y}|\boldsymbol{\Phi},\boldsymbol{\omega},\boldsymbol{\alpha}, \beta  \right)$ and $p\left( \boldsymbol{\omega}|\boldsymbol{\alpha}  \right)$, $\ln {{Q}^{\ast}_{\boldsymbol{\omega}}}\left( \boldsymbol{\omega} \right)$ is given by
		\begin{align} \label{27a}
			\begin{split}
				&\ln {{Q}^{\ast}_{\boldsymbol{\omega}}}\left( \boldsymbol{\omega} \right)			=			\mathbb{E}_{\beta}[\beta] \boldsymbol{\omega}^{\top}\sum\limits_{k=1}^{N} {\Phi}^{\top}_k\boldsymbol{y} + \operatorname { c o n s t }\\
				&-\frac{1}{2}\boldsymbol{\omega}^{\top}\left( \operatorname{diag} \mathbb{E}_{\boldsymbol{\alpha}}[  \alpha_m  ] + \mathbb{E}_{\beta}[\beta] \sum\limits_{k=1}^{N} {\Phi}^{\top}_k{\Phi}_k \right) \boldsymbol{\omega}
			\end{split}
		\end{align}

		It can be seen that the right-hand side of \eqref{27a} is a quadratic function of $\boldsymbol{\omega}$, so ${{Q}^{\ast}_{\boldsymbol{\omega}}}\left( \boldsymbol{\omega} \right)$ can be identified as a Gaussian distribution.
		Using the technique of completing the square, the mean and variance of this Gaussian distribution can be obtained as
		\begin{align}
			\begin{split}
				\boldsymbol{\Sigma}_{\boldsymbol{\omega}}&= \left( \operatorname{diag} \mathbb{E}_{\boldsymbol{\alpha}}[  \alpha_m  ] + \mathbb{E}_{\beta}[\beta] \sum\limits_{k=1}^{N} {\Phi}^{\top}_k {\Phi}_k \right) ^{-1}\\ \boldsymbol{\mu}_{\boldsymbol{\omega}} &= \mathbb{E}_{\beta}[\beta] \boldsymbol{\Sigma}_{\boldsymbol{\omega}}\sum\limits_{k=1}^{N} {\Phi}^{\top}_k\boldsymbol{y}.
			\end{split}
		\end{align}

\end{proof}

\begin{corollary}  \label{co2}
	Following the Theorem \ref{th2}, ${{Q}^{\ast}_{\boldsymbol{\alpha}}}\left( \boldsymbol{\alpha} \right)$ is a Gamma distribution
	\begin{align}
		{{Q}^{\ast}_{\boldsymbol{\alpha}}}\left( \boldsymbol{\alpha} \right) =\prod\limits_{\text{i}=1}^M{\text{Gam}(\alpha_m|a_m,b_m)}
	\end{align}
	where $a_m = a_0+\frac{1}{2}$ and $b_m = b_0+\frac{1}{2}{{\mathbb{E}}_{\boldsymbol{\omega}}}[{{\omega}_m}^{\top}{{\omega}_m}]$. $a_m$  and $b_m$ are parameters of the Gamma distribution.
\end{corollary}

\begin{proof}

		According to Theorem \ref{th2},  the logarithmic form of the optimal variational distribution $ {{Q}^{\ast}_{\boldsymbol{\alpha}}}\left( \boldsymbol{\alpha} \right) $ can be given by applying \eqref{18}:
		\begin{align}  \label{24}
			\ln {{Q}^{\ast}_{\boldsymbol{\alpha}}}\left( \boldsymbol{\alpha} \right)&={{\mathbb{E}}_{\boldsymbol{\omega},\beta }}\left[ \ln p\left( \mathbf{y},\boldsymbol{\alpha}, \beta ,\boldsymbol{\omega} \right) \right]+\operatorname { c o n s t }.
		\end{align}

		Only those terms that have some functional dependence on $\boldsymbol{\alpha}$ need to be retained, all other terms can be absorbed into the normalization constant.
		$\ln {{Q}^{\ast}_{\boldsymbol{\alpha}}}\left( \boldsymbol{\alpha} \right)$ can be further simplified as
		\begin{align} \label{25}
			\begin{split}
				\ln {{Q}^{\ast}_{\boldsymbol{\alpha}}}\left( \boldsymbol{\alpha} \right)
				& ={{\mathbb{E}}_{\boldsymbol{\omega},\beta }}[\text{ln}\ p(\boldsymbol{\omega}|\boldsymbol{\alpha} )]+{{\mathbb{E}}_{\boldsymbol{\omega},\beta }}[\text{ln}\ p(\boldsymbol{\alpha} )]+\ \operatorname { c o n s t } \\ 
				& =\text{ln}\ p(\boldsymbol{\alpha} )+{{\mathbb{E}}_{\boldsymbol{\omega},\beta}}[\text{ln}\ p(\boldsymbol{\omega}|\boldsymbol{\alpha} )]+ \operatorname { c o n s t } \\ 
			\end{split}
		\end{align}
		where  $p\left( \boldsymbol{\omega}|\boldsymbol{\alpha}  \right)$ and $p(\boldsymbol{\alpha}) $ are calculated through \eqref{6} and \eqref{7}, respectively. 
		Expanding the logarithmic form of $p\left( \boldsymbol{\omega}|\boldsymbol{\alpha}  \right)$ and $p(\boldsymbol{\alpha}) $, $	\ln {{Q}^{\ast}_{\boldsymbol{\alpha}}}\left( \boldsymbol{\alpha} \right)$ is given by
		\begin{align}\label{32a}
			\begin{split}
				\ln {{Q}^{\ast}_{\boldsymbol{\alpha}}}\left( \boldsymbol{\alpha} \right)			& =\underbrace{({{a }_{0}}-1)\text{ln}\ \boldsymbol{\alpha} -{{b }_{0}}\boldsymbol{\alpha} }_{\text{Gamma PDF}}\\
				&+\underbrace{\frac{M}{2}\text{ln}\ \boldsymbol{\alpha} -\frac{\boldsymbol{\alpha} }{2}{{\mathbb{E}}_{\boldsymbol{\omega}}}[{{\boldsymbol{\omega}}^{\top}}\boldsymbol{\omega}]}_{\text{Gaussian PDF}} 
				+ \operatorname { c o n s t }\\ 
				& =\underbrace{({{a }_{0}}+\frac{M}{2}-1)}_{a_m \ \text{parameter}} \ln \boldsymbol{\alpha} \\
				&-(\underbrace{{{b }_{0}}+\frac{1}{2}{{\mathbb{E}}_{\boldsymbol{\omega}}}[{{\boldsymbol{\omega}}^{\top}}\boldsymbol{\omega}]}_{b_m \ \text{parameter}})\boldsymbol{\alpha} 
				+ \operatorname { c o n s t }    .     \\ 
			\end{split}
		\end{align}

		It can be seen that the right-hand side of \eqref{32a} is the logarithm of a gamma distribution, so ${{Q}^{\ast}_{\boldsymbol{\alpha}}}\left( \boldsymbol{\alpha} \right)$ can be identified as a gamma distribution.
		The two parameters of this Gamma distribution can be obtained as
		\begin{align}
			\begin{split}
				a_m &= a_0+\frac{1}{2} \\
				b_m &= b_0+\frac{1}{2}{{\mathbb{E}}_{\boldsymbol{\omega}}}[{{\omega}_m}^{\top}{{\omega}_m}].
			\end{split}
		\end{align}

\end{proof}

\begin{corollary} \label{co3}
	Following the Theorem \ref{th2}, ${Q}^{\ast}_{\beta}\left( \beta \right)$ is a gamma distribution
	\begin{align}
		{Q}^{\ast}_{\beta}\left( \beta \right)={\text{Gam}(\beta|c,d)}
	\end{align}
	where $c = c_0+\frac{N}{2}$ and $d = {{d}_{0}}+\frac{1}{2}\sum\limits_{k=1}^{N}{{y}_{k}^{2}}-{{\mathbb{E}}_{\boldsymbol{\omega }}}{{[\boldsymbol{\omega }]}^{\top }}\sum\limits_{k=1}^{N}{{{{\Phi} }_{k}}{{y}_{k}}}+\frac{1}{2}\sum\limits_{k=1}^{N}{{\Phi} _{k}^{\top }{{\mathbb{E}}_{\boldsymbol{\omega }}}[{{\boldsymbol{\omega }}^{\top }}\boldsymbol{\omega }]{{{\Phi} }_{k}}}$. $c$  and $d$ are parameters of the Gamma distribution.
\end{corollary}

\begin{proof}
		According to Theorem \ref{th2},  the logarithmic form of the optimal variational distribution $ {Q}^{\ast}_{\beta}\left( \beta \right) $ can be given by applying \eqref{18}:
		\begin{align}  \label{28}
			\ln {Q}^{\ast}_{\beta}\left( \beta \right)&={{\mathbb{E}}_{\boldsymbol{\omega},\boldsymbol{\alpha} }}\left[ \ln p\left( \boldsymbol{y},\boldsymbol{\alpha}, \beta ,\boldsymbol{\omega} \right) \right] +\operatorname { c o n s t }.
		\end{align}

		Only those terms that have some functional dependence on $\beta$ need to be retained, all other terms can be absorbed into the normalization constant.
		$\ln {Q}^{\ast}_{\beta}\left( \beta \right)$ can be further simplified as
		\begin{align}  \label{29}
					\begin{split} 
				&\ln {Q}^{\ast}_{\beta}\left( \beta \right)
				={{\mathbb{E}}_{\boldsymbol{\omega},\boldsymbol{\alpha} }}\left[ \ln p\left( \boldsymbol{y}| \boldsymbol{\Phi}, \boldsymbol{\omega} ,\boldsymbol{\alpha},\beta \right) \right]\\
				&+ {{\mathbb{E}}_{\boldsymbol{\omega},\boldsymbol{\alpha} }}\left[ \ln p(\beta) \right]  +\operatorname { c o n s t }\\
				&=\ln p(\beta) + {{\mathbb{E}}_{\boldsymbol{\omega},\boldsymbol{\alpha} }}\left[ \ln p\left( \boldsymbol{y}| \boldsymbol{\Phi}, \boldsymbol{\omega} ,\boldsymbol{\alpha},\beta \right) \right] +\operatorname { c o n s t }
			\end{split}
		\end{align}
		where $p\left( \boldsymbol{y}| \boldsymbol{\Phi}, \boldsymbol{\omega} ,\boldsymbol{\alpha},\beta \right)$ and $p(\beta)$  are calculated through \eqref{5} and \eqref{8}, respectively.
		Expanding the logarithmic form of $p\left( \boldsymbol{y}| \boldsymbol{\Phi}, \boldsymbol{\omega} ,\boldsymbol{\alpha},\beta \right)$ and $p(\beta)$, $		\ln {Q}^{\ast}_{\beta}\left( \beta \right)$ is given by
		\begin{align} \label{37a}
			\begin{split}
				&\ln {Q}^{\ast}_{\beta}\left( \beta \right)
				 = \underbrace{\left( {{c}_{0}}-1+\frac{N}{2} \right)}_{c \ \text{parameter}} \ln \beta  \\
				&- \underbrace{\left( {{d}_{0}}+\frac{1}{2}\mathbb{E}\left[ { \sum\limits_{\text{k}=1}^N {\left( {{y}_{k}}-{{\boldsymbol{\omega }}^{\top }}\sum\limits_{k=1}^{N} {\Phi}_k \right)}^{2}} \right]    \right) }_{d \ \text{parameter}} \beta + \operatorname { c o n s t }  .\\ 
			\end{split}
		\end{align}

		It can be seen that the right-hand side of \eqref{37a} is the logarithm of a gamma distribution, so ${Q}^{\ast}_{\beta}\left( \beta \right)$ can be identified as a gamma distribution.
		The two parameters of this Gamma distribution can be obtained as
		\begin{align}
			\begin{split}
				c &= c_0+\frac{N}{2}\\
				d &= {{d}_{0}}+\frac{1}{2}\sum\limits_{k=1}^{N}{{y}_{k}^{2}}-{{\mathbb{E}}_{\boldsymbol{\omega }}}{{[\boldsymbol{\omega }]}^{\top }}\sum\limits_{k=1}^{N}{{{{\Phi} }_{k}}{{y}_{k}}}\\
				&+\frac{1}{2}\sum\limits_{k=1}^{N}{{\Phi} _{k}^{\top }{{\mathbb{E}}_{\boldsymbol{\omega }}}[{{\boldsymbol{\omega }}^{\top }}\boldsymbol{\omega }]{{{\Phi} }_{k}}}.
			\end{split}
		\end{align}

\end{proof}

Applying Theorem \ref{th1}, ELBO $\mathcal{L}(Q(\boldsymbol{\Theta}))$ is guaranteed to converge because it is convex with respect to each of $q^{\ast}_j(\Theta_j)$.  $\mathcal{L}(Q(\boldsymbol{\Theta}))$ can be written more explicitly as 
\begin{align}\label{39a}
	\begin{split}
		&\mathcal{L}(Q(\boldsymbol{\Theta}))={{\mathbb{E}}_{\boldsymbol{\Theta}}} \left[ \ln p(\boldsymbol{y},\boldsymbol{\omega},\boldsymbol{\alpha},\beta)\right] - {{\mathbb{E}}_{\boldsymbol{\Theta}}}\left[ \ln Q(\boldsymbol{\boldsymbol{\omega}},\boldsymbol{\alpha},\beta) \right] \\
		& =  {{\mathbb{E}}_{\boldsymbol{\omega},\beta}}[\ln p(\boldsymbol{y}|\boldsymbol{\Phi},\boldsymbol{\omega},\boldsymbol{\alpha},\beta)]+{{\mathbb{E}}_{\boldsymbol{\omega},\boldsymbol{\alpha}  }}[\ln  p(\boldsymbol{\omega}|\boldsymbol{\alpha} )]\\
		&+{{\mathbb{E}}_{\boldsymbol{\alpha}  }}[\ln p(\boldsymbol{\alpha} )]+{{\mathbb{E}}_{\beta }}[\ln p(\beta )]\\
		&-{{\mathbb{E}}_{\boldsymbol{w}}}[\ln Q_{\boldsymbol{\omega}}(\boldsymbol{\omega})]
		-{{\mathbb{E}}_{\boldsymbol{\alpha} }}[\ln Q_{\boldsymbol{\alpha}}(\boldsymbol{\alpha} )]-{{\mathbb{E}}_{\beta}}[\ln Q_{\beta}(\beta )].
	\end{split}
\end{align}

All terms in \eqref{39a} can be evaluated as follow
\begin{small}
	\begin{align} \label{32}
		\begin{split}
			&\mathcal{L}(Q(\boldsymbol{\Theta})) =\frac{N}{2}\left( \mathbb{E}_{\beta}[\ln \beta] - \ln (2\pi)\right)
			-\frac{1}{2}\mathbb{E}_{\beta}[\beta] \\
			&\cdot \left( \sum\limits_{\text{k}=1}^N {y}^2_k -2\mathbb{E}_{\boldsymbol{\omega}}[\boldsymbol{\omega}]^{\top}\sum\limits_{k=1}^{N}{{{{\Phi} }_{k}}{{y}_{k}}}+\frac{1}{2}\sum\limits_{k=1}^{N}{{\Phi} _{k}^{\top }{{\mathbb{E}}_{\boldsymbol{\omega }}}[{{\boldsymbol{\omega }}^{\top }}\boldsymbol{\omega }]{{{\Phi} }_{k}}}       \right) \\
			&-\frac{N+1}{2}\ln (2\pi) 
			-\frac{1}{2} \sum\limits_{\text{i}=0}^M \left( \mathbb{E}_{\boldsymbol{\alpha}}[\ln \alpha_m] + \mathbb{E}_{\boldsymbol{\alpha}}[ \alpha_m] \mathbb{E}_{\boldsymbol{\omega}}[ {\omega}_m^{\top}{\omega}_m]\right)  \\
			& +(N+1)(a_0 \ln b_0 - \ln \Gamma(a))   + (a_0 -1) \sum\limits_{\text{i}=1}^M \mathbb{E}_{\boldsymbol{\alpha}}[\ln \alpha_m] 
			\\&-b_0 \sum\limits_{\text{i}=1}^M \mathbb{E}_{\boldsymbol{\alpha}}[\alpha_m] 	+c_0 \ln d_0 +(c_0-1)\mathbb{E}_{\boldsymbol{\alpha}}[\ln \alpha_m] - d_0 \mathbb{E}_{\boldsymbol{\alpha}}[ \alpha_m] \\
			&- \ln \Gamma(c_0)  -\frac{1}{2}\left( (N+1)(1+\ln(2\pi))-\ln|\boldsymbol{\Sigma}_{\boldsymbol{\omega}} |    \right)  \\
			&	- \sum\limits_{\text{i}=1}^M \left( a_m\ln b_m+ (a_m -1)\mathbb{E}_{\boldsymbol{\alpha}}[\ln \alpha_m] + b_m \mathbb{E}_{\boldsymbol{\alpha}}[\alpha_m] -\ln \Gamma(a_m) \right)
			\\&-c \ln d - (c-1) \mathbb{E}_{\beta}[\ln \beta] +d \mathbb{E}_{\beta}[\beta]+\ln \Gamma(c) 
		\end{split}
	\end{align}
\end{small}
where $\Gamma$ denotes the Gamma function.

In summary, approximate distribution $Q(\boldsymbol{\omega},\boldsymbol{\alpha},\beta)$ can now be calculated by iteratively computing ${{Q}^{\ast}_{\boldsymbol{\omega}}}\left( \boldsymbol{\omega} \right)$, ${{Q}^{\ast}_{\boldsymbol{\alpha}}}\left( \boldsymbol{\alpha} \right)$ and ${Q}^{\ast}_{\beta}\left( \beta \right)$.
At each iteration, $\mathcal{L}(Q(\boldsymbol{\Theta}))$ can be computed by \eqref{32}.
The required moments in \eqref{32} are easily evaluated using basic properties of probability distributions
\begin{align}
	\begin{split}
		&\mathbb{E}_{\boldsymbol{\omega}}[\boldsymbol{\omega}]= \boldsymbol{\mu}_{\boldsymbol{\omega}}  , \ \ \ \ \
		\mathbb{E}_{\boldsymbol{\omega}}[\boldsymbol{\omega}\boldsymbol{\omega}^{\top}]=\boldsymbol{\Sigma}_{\boldsymbol{\omega}} +  \boldsymbol{\mu}_{\boldsymbol{\omega}}\boldsymbol{\mu}^{\top}_{\boldsymbol{\omega}},\\
		&\mathbb{E}_{\boldsymbol{\alpha}}[  \alpha_m  ] = \frac{a_m}{b_m} ,\ \ \
		\mathbb{E}_{\boldsymbol{\alpha}}[ \ln \alpha_m  ]=  \psi(a_m) - \ln b_m ,        \\
		&\mathbb{E}_{\beta}[\beta] = \frac{c}{d} ,\ \ \ \ \ \ \
		\mathbb{E}_{\beta}[ \ln \beta  ]=  \psi(c) - \ln d  ,       \\
	\end{split}
\end{align}
where  $\psi(a) := \frac{d}{da}\ln \Gamma(a)$ denotes the digamma function.

	\begin{remark}
		Assumption \ref{qd} comes from statistical physics, where it is known as mean-field theory \citep{chaikin1995principles}.
	\end{remark}
	\begin{remark}
		With respect to Theorem \ref{th2}, it is more convenient in practice to work with the form \eqref{18} and then recover the normalization constant by inspection. The same approach is used for the three Corollaries. It should be noted that \eqref{18} does not satisfy all cases, and some special cases must use \eqref{12a}.
	\end{remark}
	\begin{remark}
		Theorem \ref{th1} and Theorem \ref{th2} are resummarized from papers \citep{jordanIntroductionVariationalMethods1999a,beal2003variational,bishop2006pattern}, which focus more on the origins of variational reasoning, as well as the derivation of the formulas and the algorithms themselves. 
		To describe variational inference more directly, we resummarized them as \ref{th1} and Theorem \ref{th2}. In contrast to existing work, Theorem \ref{th1} and Theorem \ref{th2} focus more on applications of variational inference.
	\end{remark}
	\begin{remark}
		In the proof of Corollary \ref{co1}. It is worth emphasizing that ${{Q}^{\ast}_{\boldsymbol{\omega}}}\left( \boldsymbol{\omega} \right)$ is not assumed to be Gaussian, but rather this result is derived by variational optimizing the KL divergence of all possible distributions ${{Q}_{\boldsymbol{\omega}}}\left( \boldsymbol{\omega} \right)$. Moreover, the additive constant does not need to be considered explicitly, since it represents the normalization constant. The same holds for Corollary \ref{co2} and Corollary \ref{co3}.
	\end{remark}
	\begin{remark}
		The computation of ELBO is not necessary for the inference problem, but it provides a check that the algorithm and the theory behind it are correct.  This is due to the fact that every update of the optimal variational distribution ensures an increase in the ELBO.
	\end{remark}
	\begin{remark}
		EM is commonly used to learn the parameters. However, the EM method can only estimate the maximum posterior for each parameter, representing the single most probable value,
		therefore does not take the uncertainty of the parameters into account. 
		Variational Bayesian inference is used in this study to learn unknown nonlinear systems and infer the future evolution of the system.
	\end{remark}

\section{Model predictive control with stability } \label{sec:mpc}

The parameter $\boldsymbol{\Theta}$ of NARX model in \eqref{5} is approximated by the optimal variational distribution $q^{\ast}_j(\Theta_j)$ in Section II. 
The prediction for a new input can then be obtained by computing the predictive distribution at the $k+1$ moment. Thus, given the input-output training data $\mathcal{D}$, the new output is an evaluation of the prediction distribution $p({{y}}_{k+1}|{\Phi}_{k+1},\mathcal{D}) $
\hspace{-10pt}\begin{small}
\hspace{-10pt}\begin{align} \label{34}
	\begin{split}
		&p({{y}}_{k+1}|{\Phi}_{k+1},\mathcal{D}) \\
		&= \int_{\boldsymbol{\omega}}\int_{\boldsymbol{\alpha}}\int_{\beta}   p({y}_{k+1}|{\Phi}_{k+1},\boldsymbol{\omega},\boldsymbol{\alpha},\beta) p(\boldsymbol{\omega},\boldsymbol{\alpha},\beta|\mathcal{D})
		d{\boldsymbol{\omega}}d{\boldsymbol{\alpha}}d{\beta} \\
		&\simeq   \int_{\boldsymbol{\omega}}\int_{\boldsymbol{\alpha}}\int_{\beta} p({{y}}_{k+1}|{\Phi}_{k+1},\boldsymbol{\omega},\boldsymbol{\alpha},\beta) Q_{\boldsymbol{\omega}}(\boldsymbol{\omega})Q_{\boldsymbol{\alpha}}(\boldsymbol{\alpha})Q_{\beta}(\beta)	d{\boldsymbol{\omega}}d{\boldsymbol{\alpha}}d{\beta} \\
		&=\int_{\boldsymbol{\omega}}\int_{\beta}    \mathcal{N}({{y}}_{k+1}|\boldsymbol{\omega}^{\top}{\Phi}_{k+1},\beta^{-1})  \mathcal{N}(\boldsymbol{\omega}|\boldsymbol{\mu}_{\boldsymbol{\omega}},\boldsymbol{\Sigma}_{\boldsymbol{\omega}})        \text{Gam}(\beta|c,d)	d{\boldsymbol{\omega}}d{\beta} \\
		&=\int_{\beta}   \mathcal{N}({{y}}_{k+1}|\boldsymbol{\mu}_{\boldsymbol{\omega}}^{\top}{\Phi}_{k+1}, \beta^{-1}(1+ {\Phi}_{k+1}^{\top} \boldsymbol{\Sigma}_{\boldsymbol{\omega}}{\Phi}_{k+1} ))     \text{Gam}(\beta|c,d)  d\beta \\
		& = \text{St}({{y}}_{k+1}|\boldsymbol{\mu}_{\boldsymbol{\omega}}^{\top}{\Phi}_{k+1}, (1+ {\Phi}_{k+1}^{\top} \boldsymbol{\Sigma}_{\boldsymbol{\omega}}{\Phi}_{k+1} )^{-1} \frac{a_m}{b_m}, 2a_m ) .
	\end{split}
\end{align}
\end{small}
where $\text{St}$ denotes Student's t distribution.  $\boldsymbol{\mu}_{\boldsymbol{\omega}}^{\top}{\Phi}_{k+1}$,  $(1+ {\Phi}_{k+1}^{\top} \boldsymbol{\Sigma}_{\boldsymbol{\omega}}{\Phi}_{k+1} ))^{-1} \frac{a_m}{b_m}$, and $2a_m$ 
	are mean, precision and degrees of freedom for this Student's t distribution, respectively.

As the degrees of freedom $2c$ tends to infinity, the Student's t distribution can be reduced to a Gaussian distribution with mean $\boldsymbol{\mu}_{\boldsymbol{\omega}}^{\top}\boldsymbol{\Phi}_{k+1}$, precision $(1+ \boldsymbol{\Phi}_{k+1}^{\top} \boldsymbol{\Sigma}_{\boldsymbol{\omega}}\boldsymbol{\Phi}_{k+1} ))^{-1} \frac{c}{d}$. As a generalization of the Gaussian distribution, the Student¡¯s t distribution can be regarded as a mixture of an infinite number of Gaussian distributions with the same mean and different variances, and the maximum likelihood function of the Student¡¯s t distribution is robust to outliers.

The prediction distribution in \eqref{34} provides the mean and variance as follows
\begin{align} \label{35}
	\begin{split}
		\mathbb{E}[y_{k+1}] &= \boldsymbol{\mu}_{\boldsymbol{\omega}}^{\top}{\Phi}_{k+1} ,\\
		\text{Var}[y_{k+1}] &= (1+ {\Phi}_{k+1}^{\top} \boldsymbol{\Sigma}_{\boldsymbol{\omega}}{\Phi}_{k+1} )\frac{b_m}{a_m-1} .
	\end{split}
\end{align}

Following the prediction result of variational inference in \eqref{34}, the nonlinear dynamic system given in \eqref{2} can be rewritten as
\begin{align}\label{36}
	{y}_{k+1} = f\left( \boldsymbol{x}_k,\boldsymbol{u}_k; \boldsymbol{\Theta}, \mathcal{D} \right) +{\xi}_{k}= \mathbb{E}[y_{k+1}] .
\end{align}

Given a specific state of the plant and a sequence of forthcoming control inputs at time step $k$, the new state $	\boldsymbol{x}_{k+1}$ of the nonlinear dynamic system given in \eqref{2} can be expressed as
\begin{align} \label{37}
	\begin{split}
		\boldsymbol{x}_{k+1} =&F \left( {\boldsymbol{x}}_{ k},{\boldsymbol{u}}_{k}; \boldsymbol{\Theta},\mathcal{D} \right) \\
		= &( f\left( \boldsymbol{x}_k,\boldsymbol{u}_k; \boldsymbol{\Theta}, \mathcal{D} \right)+\xi_{k}\\
		&, 
		{y}_k, \cdots, {y}_{k-n_a+1}, {u}_k, \cdots , {u}_{k-n_b+1}
		)  .
	\end{split}
\end{align}

Let $\boldsymbol{y}_{j \vert k}$  denotes at time $k$  predict the measurements output at time $k + j$  based on the prospective control sequence $\boldsymbol{u}_{k+j}$, $\hat{\boldsymbol{x}}_{j \vert k}$ denotes the predicted state at time $j$ given the measurements at time $k$.
The prediction model applied in the nonlinear MPC is deduced from the prediction function given in \eqref{37} as following:
\begin{align} \label{43}
	\begin{split}
		\hat{\boldsymbol{x}}_{j+1 \vert k}  
		&=F \left( \hat{\boldsymbol{x}}_{j \vert k},\hat{\boldsymbol{u}}_{k+j}; \boldsymbol{\Theta},\mathcal{D} \right) \\
		&=( f\left( \hat{\boldsymbol{x}}_{j \vert k},\boldsymbol{u}_{k+j}; \boldsymbol{\Theta}, \mathcal{D} \right) +{\xi}_{k},\hat{{y}}_{j \vert k},\cdots ,{{y}}_{k},
		\cdots \\
		&, {{y}}_ {k+j-{{n}_{a}}+1} ,\cdots ,{{u}}_{k+j} ,\cdots ,{{u}}_{k+j-{{n}_{b}}+1} ) \\
		&=( \mathbb{E}[y_{j+1 \vert k}],\hat{{y}}_{j \vert k},\cdots ,{{y}}_{k},\cdots \\
		&,{{y}}_ {k+j-{{n}_{a}}+1} ,
		\cdots ,{{u}}_{k+j} ,\cdots ,{{u}}_{k+j-{{n}_{b}}+1} ).\\
	\end{split}
\end{align}
where the predicted state \\$\hat{\boldsymbol{x}}_{j \vert k} = \left( \hat{{y}}_{j \vert k},\cdots ,\hat{{y}}_{1 \vert k} ,{{y}}_{k},\cdots ,{{y}}_ {k+j-{{n}_{a}}} ,\cdots ,{{u}}_{k+j-1}  ,\cdots ,{{u}}_{k+j-{{n}_{b}}}
\right) $ includes real measurements ${y}$ or ${u}$ if $n_a\geq j$ or $n_b\geq j$, respectively, otherwise it includes only the estimated  $\hat{y}$ or $\hat{u}$.

A positive stage cost function $\ell(\boldsymbol{x}, \boldsymbol{u})$ is introduced as follows to deduce  the predictive controller
\begin{align} \label{47a}
	\ell \left( {\boldsymbol{x}},\boldsymbol{u} \right) = \ell_s(\boldsymbol{x}-\boldsymbol{x}_r,\boldsymbol{u}-\boldsymbol{u}_r)+\ell_b(\boldsymbol{y}),
\end{align}
where $\ell_s(\boldsymbol{x}-\boldsymbol{x}_r,\boldsymbol{u}-\boldsymbol{u}_r)$ penalizes the tracking error of inputs and state deviation with respect to the reference that is provided by $(\boldsymbol{x}_r, \boldsymbol{u}_r)$. 
Additionally, $\ell_b(\boldsymbol{y})$ represents a barrier function responsible for implementing a flexible constraint on the output values, which is defined by
\begin{align}
	\ell_b(\boldsymbol{y}) \geq   \alpha_b(d(\boldsymbol{y},\mathcal{Y}))      \ , {\forall} \boldsymbol{y}
\end{align}
where the barrier cost function $\ell_b(\cdot)$ should adhere to the condition that $\ell_b(\boldsymbol{y}) = 0, \forall \boldsymbol{y} \in \mathcal{Y}$, $\alpha_b(\cdot)$ is denoted the  $\mathcal{K}$ functions,  $d(\cdot)$ represents  the distance function.

\begin{assumption} \label{ass-4-4a}
	It is assumed $\ell \left( {\boldsymbol{x}},\boldsymbol{u} \right)$ to be continuous, then 
	two $\mathcal{K}$ functions, denoted as $\alpha_x$ and $\alpha_u$, exist such that:
	\begin{align}
		\ell \left( {\boldsymbol{x}},\boldsymbol{u} \right) \geq \alpha_x(|| \boldsymbol{x}-\boldsymbol{x}_r||)+\alpha_u(||\boldsymbol{u}-\boldsymbol{u}_r||).
	\end{align}
\end{assumption} 

With the prediction model \eqref{43}, the objective of the nonlinear MPC is to solve the following finite-horizon optimal regulation problem
    \begin{align} \label{42}
	\begin{split}
		&\underset{\hat{\boldsymbol{u}}}{\mathop{\min }}\,\mathcal{J}_{N_c,N_p}\left( \hat{\boldsymbol{x}}_k,\hat{\boldsymbol{u}} \right)\\
		=&\underset{i=0}{\overset{N_c-1}{\mathop{\sum }}}\, \ell \left( \hat{\boldsymbol{x}}_{i \vert k},\boldsymbol{u}_i \right)+ \underset{i=N_c}{\overset{N_p-1}{\mathop{\sum }}}\, \ell \left( \hat{\boldsymbol{x}}_{i \vert k},\kappa_f(\hat{\boldsymbol{x}}_{i \vert k})\right)  \\
		&+\lambda {{V}_{f}}\left(\boldsymbol{x}_{N \vert k} - \boldsymbol{x}_r \right)  \\
		\mathop{s.t.}\ &\underbrace{\hat{\boldsymbol{x}}_{j+1 \vert k}  =F \left( \hat{\boldsymbol{x}}_{j \vert k},\hat{\boldsymbol{u}}_{j}; \boldsymbol{\Theta},\mathcal{D} \right)}_{\mathsf{Nonlinear\ System \ Model}}  ,    \ j=0,\cdots,N-1,   \\
		& \hat{ \boldsymbol{x}}_{0 \vert k} = \boldsymbol{x}_k , \\
		& \hat{\boldsymbol{x}}_{j \vert k} \in         \mathcal{X}  ,  \ j=0,\cdots,N_p-1, \\
		& \hat{\boldsymbol{u}}_{j}  \in  \mathcal{U} , \ j=0,\cdots,N_p-1, \\
		& \hat{\boldsymbol{x}}_{N \vert k}   \in \mathcal{X}_{N} ,
	\end{split}
\end{align}
where $\lambda \geq 1$ represents a design parameter for the controller, $N_p$ represents the prediction horizon, and $N_c \leq N_p$ is the control horizon, ${{V}_{f}}(\cdot)$ serves as  the terminal cost function, $\mathcal{X}_{N} $ defines the terminal region.
Achieving asymptotic stability is guaranteed through the design of a terminal control law $\kappa_f(\cdot)$.

Terminal constraints are omitted because the terminal cost function $V_f$ is scaled by $\lambda \geq 1$.

	\begin{definition}
		The closed loop system\\ $\boldsymbol{x}_{k+1}=F(\boldsymbol{x}_k,\kappa_{f}(\boldsymbol{x};\boldsymbol{\Theta},\mathcal{D});\boldsymbol{\Theta},\mathcal{D})$ is ISS, if there exists a class $\mathcal{KL}$ function $\alpha_x$ and a class $\mathcal{K}$ function $\alpha_r$ such that
		\begin{align}
			|| \boldsymbol{x}_k - \boldsymbol{x}_r || \leq \alpha_x (|| \boldsymbol{x}_0 - \boldsymbol{x}_r ||) + \alpha_r\left( \underset{k \geq 0}{\max} || \boldsymbol{e}_k ||\right) 
		\end{align}
		holds for all initial states $\boldsymbol{x}_0$, errors $\boldsymbol{e}_k$, and for all $k$.
		$\kappa_{f}(\boldsymbol{x};\boldsymbol{\Theta},\mathcal{D})$ is the nominal controller resulting from the optimization problem \eqref{42} \citep{khalil2002nonlinear}.
	\end{definition}

ISS demonstrates the bounded effect of uncertainty, where the uncertain system converges asymptotically to the reference point if the error asymptotically dissipates. 
	Therefore, we first establish stability in the nominal case by assuming that the prediction error is zero, i.e., the true system is the same as the nominal model. Then when the prediction error is non-zero but bounded, we establish robust stability under ISS.

\begin{assumption}\label{ass-4-2}
	It is assumed there exists a terminal controller $\kappa_f(\cdot):\mathbb{R}^{n_x}\rightarrow \mathbb{R}^{n_u}$, a terminal cost function $V_f(\cdot):\mathbb{R}^{n_x}\rightarrow \mathbb{R}$ and a region $\Omega_\vartheta=\left\lbrace \boldsymbol{x}:V_f(\boldsymbol{x})\leq \vartheta \right\rbrace $, with $\vartheta$ representing a positive constant.  For all $\boldsymbol{x} \in \Omega_\vartheta$, the following conditions hold:
	\begin{align}
		\begin{split}
			&\alpha_1(|| \boldsymbol{x}-\boldsymbol{x}_r  ||) \leq V_f(\boldsymbol{x}-\boldsymbol{x}_r) \leq \alpha_2(|| \boldsymbol{x}-\boldsymbol{x}_r ||) \\
			&V_f(\boldsymbol{x}^{+} -\boldsymbol{x}_r) -  V_f(\boldsymbol{x}-\boldsymbol{x}_r) \leq - \ell_s(\boldsymbol{x}-\boldsymbol{x}_r, \kappa_f(\boldsymbol{x})-\boldsymbol{u}_r) \\
			&\kappa_f(\boldsymbol{x})  \in \mathcal{U}
		\end{split}
	\end{align}
	where $\alpha_i$ is $\mathcal{K}_{\infty}$ functions, $\boldsymbol{x}^{+}=F(\boldsymbol{x}-\boldsymbol{x}_r,   \kappa_f(\boldsymbol{x}) -\boldsymbol{u}_r ; \mathcal{D}    )$.
\end{assumption}

\begin{assumption}\label{ass-4-6a}
	It is assumed that there exists a class $\mathcal{K}$ functions, $\alpha_y$, such that  ${\forall} \boldsymbol{u} \in \mathcal{U}$, the stage cost \eqref{47a} satisfies $\ell(\boldsymbol{x}_r,\boldsymbol{u}_r)=0$ and $\ell(\boldsymbol{x},\boldsymbol{u}_k)\geq \alpha_y(|| \boldsymbol{x}-\boldsymbol{x}_r ||)$.
\end{assumption}


\begin{theorem}
	Let Assumption \ref{ass-4-2} and Assumption \ref{ass-4-6a} hold, and let $\kappa_{f}$ be the resulting control law of the optimization problem in \eqref{42}. ${\forall} \lambda \geq 1$, there exists a region of feasibility without terminal constraint defined as $\boldsymbol{x}_{N_c,N_p}(\lambda)  $ such that  ${\forall} \boldsymbol{x} \in \boldsymbol{x}_{N_c,N_p}(\lambda)$, the closed-loop system $\boldsymbol{x}_{k+1}=F(\boldsymbol{x}_k,\kappa_{f}(\boldsymbol{x};\boldsymbol{\Theta},\mathcal{D});\boldsymbol{\Theta},\mathcal{D})$ can be asymptotically stable. 
	The size of the set $\boldsymbol{x}_{N_c,N_p}(\lambda)  $ increases with $\lambda$
	\citep{limonStabilityConstrainedMPC2006}.
\end{theorem}

\begin{proof}
		Under the conditions that Assumption \ref{ass-4-2} and Assumption \ref{ass-4-6a} are satisfied, Theorem 4 of \citep{limonStabilityConstrainedMPC2006} indicates that the controller obtained from the solution of optimization problem \eqref{42} will stabilize closed-loop systems $\boldsymbol{x}_{k+1}=F(\boldsymbol{x},\kappa_{f}(\boldsymbol{x};\boldsymbol{\Theta},\mathcal{D});\boldsymbol{\Theta},\mathcal{D})$ asymptotically in the set $\boldsymbol{x}_{N_c,N_p}(\lambda)  $. In particular, Theorem 1 of \citep{manzanoOutputFeedbackMPC2019} further proves the recursive feasibility and asymptotically stability of the closed-loop system. For detailed proofs, readers are referred to the above two papers.	
		
\end{proof}

Having established that the nominal system is asymptotically stable, the analysis of robustness continues below.
	\begin{assumption} \label{ass-4-7a}
		It is assumed that there exists a constant $\nu  \textless \infty $ such that the prediction error $\delta_k$ 
		\begin{align}
			\delta_k  = \boldsymbol{y}_{k+1} -  \hat{\boldsymbol{y}}_{k+1}
		\end{align}
		about the true system $\boldsymbol{y}_{k+1}$ and nominal systems $\hat{\boldsymbol{y}}_{k+1}$ is bounded $|\delta_k| \leq  \nu$.
\end{assumption}

Prediction errors include approximation errors and measurement noise. 
	Measurement noise is assumed to lie within a tight set and is therefore bounded. 
	In addition, as the size of the data set increases, this uncertainty decreases and the posterior probability distribution narrows, referring to Remark \ref{rem8}.  
	Therefore, if the prediction error is bounded on the training set and the training data is representative of the data seen in the closed-loop operation, Assumption \ref{ass-4-7a} is reasonable.

\begin{assumption}\label{ass-4-8a}
	It is assumed that the nominal model \\ $F \left( {\boldsymbol{x}}_{ k},{\boldsymbol{u}}_{k}; \boldsymbol{\Theta},\mathcal{D} \right)$ is uniformly continuous  with respect to ${\boldsymbol{x}}_{ k}$, $\forall \boldsymbol{x}_k \in \boldsymbol{x}_{N_c,N_p}(\lambda) $, $\forall \boldsymbol{u}_k \in \mathcal{U} $.
	Additionally, it is assumed that both the stage cost function $\ell \left( {\boldsymbol{x}_k},\boldsymbol{u}_k \right)$ and the terminal cost function $V_f(\boldsymbol{x}_k)$ exhibit uniform continuity, $\forall \boldsymbol{x}_k \in \boldsymbol{x}_{N_c,N_p}(\lambda)$, $\forall \boldsymbol{u}_k \in \mathcal{U} $.
\end{assumption}

\begin{theorem} \label{th4}
Let Assumption \ref{ass-4-4a}-\ref{ass-4-6a} hold, and let $\kappa_{f}$ be the predictive controller of the optimization problem in \eqref{42}.
		For a sufficiently small bound on the uncertainty,
		$\Omega_\delta \subseteq \boldsymbol{x}_{N_c,N_p}(\lambda)$  is a robust invariant set.
		the closed-loop system $\boldsymbol{x}_{k+1}=F(\boldsymbol{x},\kappa_{f}(\boldsymbol{x};\boldsymbol{\Theta},\mathcal{D});\boldsymbol{\Theta},\mathcal{D})$ satisfies the ISS property within the robust invariant set $\Omega_\delta $ \citep{limonInputtoStateStabilityUnifying2009}.
\end{theorem} 

\begin{proof}
		Theorem 2 is a statement of \cite[Theorem 4]{limonInputtoStateStabilityUnifying2009} within the context of the NARX model. 
		Assumption \ref{ass-4-2} and Assumption \ref{ass-4-6a} ensure the asymptotic stability of the nominal system. If Assumption \ref{ass-4-7a} and \ref{ass-4-8a} holds, Theorem 4 of \citep{limonInputtoStateStabilityUnifying2009} holds directly.
		
\end{proof}

\begin{remark} \label{rem8}
	In \eqref{35}, the first term of the variance $\text{Var}[y_{k+1}]$  indicates the uncertainty in the predictive distribution and that this uncertainty is due to noise $\beta$. The second term indicates the uncertainty due to the $\boldsymbol{\omega}$ in \eqref{26}. 
	With an expanding data set, the level of uncertainty diminishes, causing the posterior probability distribution to become more concentrated.
	It can be shown that $\text{Var}[y_{k+1}] \leq \text{Var}[y_{k}]$ \citep{qazazUpperBoundBayesian1997}.
\end{remark}

\section{PEMFC Temperature Control Case studies}\label{sec:experiments} 

Within the thermal management system of the PEMFC (Proton Exchange Membrane Fuel Cell), sensors gather data on temperature and load, transmitting it to the controller. This controller, in turn, adjusts the pump and fan speeds according to predefined control objectives. A water pump propels the coolant through the stack and heat sink, dissipating the heat generated by the stack to the surroundings.
This work is based on a fuel cell stack model rated at 30kW, consisting of 180 individual fuel cells with an optimum operating temperature point of 60 $^\circ C$ to 80  $^\circ C$. 
Our PEMFC experiments are analyzed from data collected from real devices.
The principle of the cooling system of PEMFC is shown in Figure \ref{fig.1}.
To simplify the model to achieve predictive control, the response of the system is assumed and simplified: 1. All heat generated by the stack is absorbed by the coolant; 2. The temperature is uniformly distributed across the stack and heat sink with no temperature variations; 3. Pump operation only facilitates coolant circulation and does not affect its temperature; 4. Heat dissipation exclusively occurs through the heat sink, and the coolant pipeline itself does not contribute to heat dissipation.
In addition, to simplify the controller design, the cooling water pressure is not one of the control variables. 
Therefore, the cooling water pump is not arranged between the radiator and the stack, but between the water tank and the radiator, which can prevent the cooling water from being forced into the stack.

\begin{figure}
	\centering
		\includegraphics[width=\columnwidth]{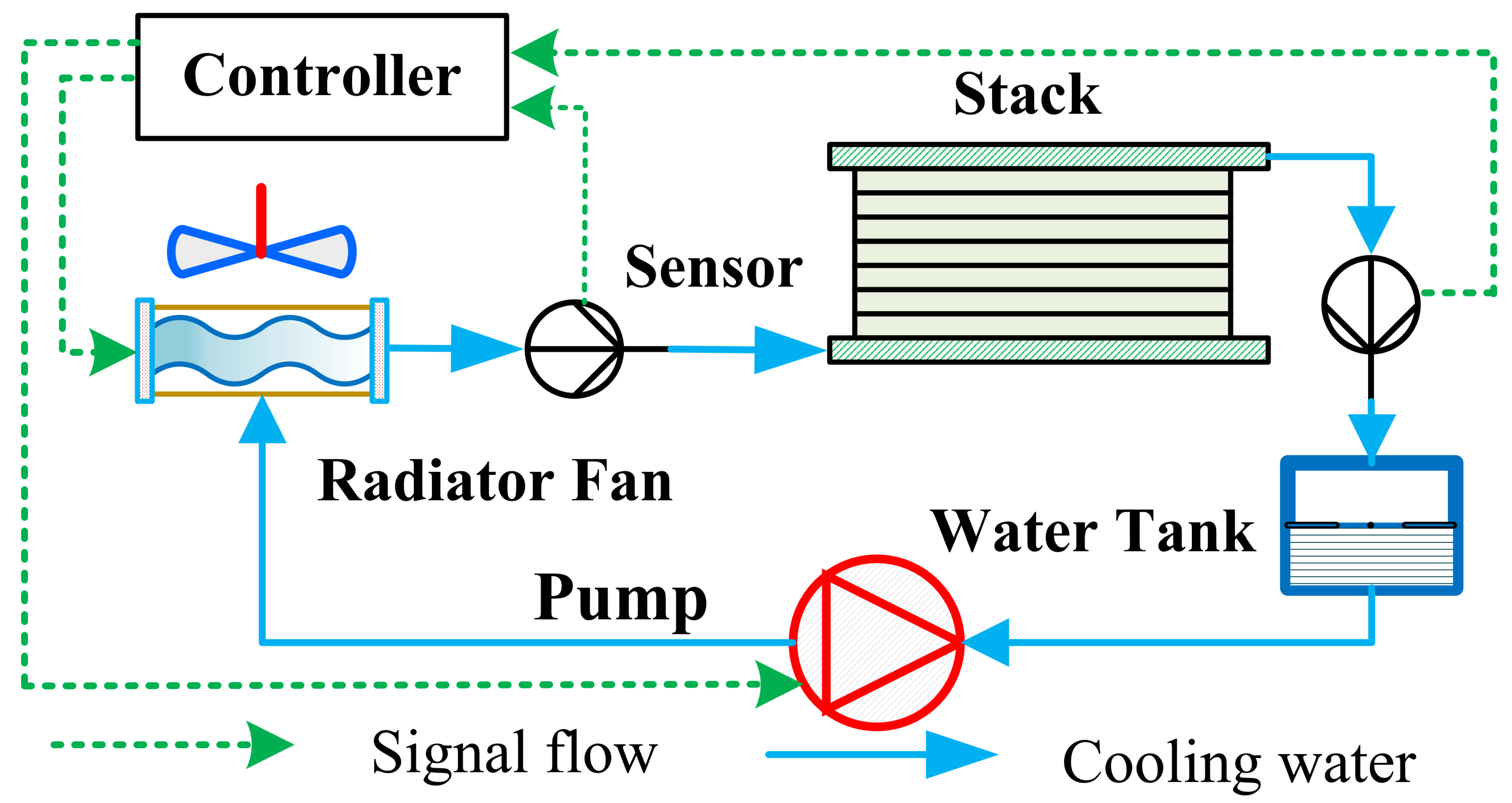}
	\caption{ Cooling system schematic.}
	\label{fig.1}
\end{figure}

The stack heat production is $Q_{gen}$
\begin{align}
	Q_{gen} = Q_{react}-P_{st}
\end{align}
where $ Q_{react}$ 
the heat produced by the electrochemical reaction within the stack, while $P_{st}$ signifies the power utilized by the stack load.

The heat exchange process between the stack and the cooling water can be calculated
\begin{align}
	\dot{T}_{w,out} = -\frac{W_w}{m_{st}}(T_{w,out}-T_{w,in})+\frac{Q_{gen}}{m_{st}c_w}
\end{align}
where $\dot{T}_{w,out}$ is the  rate of change of the stack cooling water outlet temperature, $m_{st}$ is the mass of the cooling water contained inside the stack. $W_w$ is the mass flow of cooling water through the stack. 
Meanwhile, $T_{w,out}$ and $T_{w,in}$ represent the temperatures at the outlet and inlet of the cooling water for the stack, respectively, and $c_w$ stands for the specific heat capacity of water.

The rate of change of the tank water temperature $\dot{T}_{w,t}$ is calculated as
\begin{align}
	\dot{T}_{w,tank} = \frac{W_w}{m_{tank}}(T_{w,out}-T_{w,tank})
\end{align}
where $m_{tank}$ represents the mass of water in the same tank.

According to the previous assumptions, the radiator cooling water inlet temperature matches the tank temperature $T_{w,t}$, while the radiator cooling water outlet temperature corresponds to the temperature of the stack cooling water inlet $T_{w,in}$. 
The rate of change of stack cooling water inlet temperature is calculated as follows
\begin{align}
	\dot{T}_{w,in}=\frac{W_w}{m_{rad}}(T_{w,tank}-T_{w,in}-\Delta T)
\end{align}
where $m_{rad}$ represents the mass of cooling water contained in the radiator piping. $\Delta T$ represents the temperature decreases caused by the radiator and can be calculated as follows
\begin{align}
	\Delta T = \frac{A_{rad}(T_{w,tank}-T_{amb})h_{rad}(W_a)}{c_{w}W_{w}}
\end{align}
where $A_{rad}$ represents the useful heat transfer area of the radiator. $T_{amb}$ represents the ambient  temperature. 
The $h_{rad}(W_a)$, as the heat transfer coefficient of the radiator, is a nonlinear function of the mass flow of air $W_a$ and can be calculated as
\begin{align}
	h_{rad}(W_a) = -1.4495W^2_a+5.9045W_a-0.1147.
\end{align}

Combining with the above formulas, the individual equations of the cooling system are organized into state space equations. The 
state variables are denote as $\boldsymbol{x}=[x_1,x_2,x_3]=[T_{w,tank},T_{w,in},T_{w,out}]$, the 
input is given by $\boldsymbol{u}=[u_1,u_2]=[W_w,W_a]$
. Finally, the nonlinear form of the system is obtained as follows
\begin{align} \label{51}
	\begin{split}
		\dot{x}_1 &= \frac{u_1}{m_{rad}}\left( x_1-x_2- \frac{ A_{rad}(x_1-T_{amb})h_{rad}(u_2)}{c_w u_1}\right) \\
		\dot{x}_2 &= \frac{u_1}{m_{tank}}\left( x_3-x_1\right)  \\
		\dot{x}_3 &=-\frac{u_1}{m_{st}}\left( x_3-x_2\right) +\frac{Q_{gen}}{m_{st}c_w}.
	\end{split}
\end{align}

	Model \eqref{51} is used for the simulation and generation of training data $\mathcal{D}$ with a sampling time of $ts = 0.1 s$. 
	Referring to \eqref{2}, each data point $( y_i , \boldsymbol{z}_i )$ consists of value for $\left(  {y}_{k+1}, \cdots , {y}_{k-n_a}, {u}_{k}, \cdots {u}_{k-n_b}\right) $ values, where $y_{k+1}$ is the output of NSVB and $\boldsymbol{z} = \left(  {y}_k, \cdots , {y}_{k-n_a}, {u}_{k}, \cdots {u}_{k-n_b}\right) $ is the corresponding regressor. 
	Since the obtained training data covers most of the input-output space in the closed-loop operation, it can be assumed that the bounded prediction error of the training data implies the bounded prediction error in the closed-loop operation. 
	In order to train the NSVB model, hyperparameters $ a_0, b_0, c_0$ and $ d_0 $ are initialized to $10^{-5}$ thereby making the inference process dominated by the data.
	In addition, the state is $\boldsymbol{x}_k = [y_k, y_{k-1}, y_{k-2}]$ and the regressor $\boldsymbol{z}$ is $[y_k,y_{k-1},y_{k-2},u_k]$. 
	Referring to \eqref{5} in Section \ref{sec:VB}, the parameter $\boldsymbol{\Theta}=(\boldsymbol{\omega},\boldsymbol{\alpha},\beta)$ of the NARX model is learned via variational inference.
	The target reference equilibrium is a tank temperature $T_{w,tank}$ of $336.15 K$ and  stack cooling water outlet temperature $T_{w,out}$ of $343.15K$, where $K$ is Kelvin scale. The constraints for mass flow of cooling water $W_w$ is $\mathcal{U} = \left\lbrace  0 \leq W_w \leq 60 \right\rbrace $ and the constraints for mass flow of air $W_a$ is $\mathcal{U} = \left\lbrace  0 \leq W_a \leq 2 \right\rbrace $. 
	The output data is added to the random measurement noise $\epsilon \sim  {\mathcal{N}\left( 0, 0.7 \right)}$.
	Predictive horizon equals control horizon $N_c=N_p=10$.

	The terminal controller $\kappa_{f}(\boldsymbol{x})$ is defined as $\kappa_{f}(\boldsymbol{x}) = \boldsymbol{K}^{\top}(\boldsymbol{x}-\boldsymbol{x}_r)+\boldsymbol{u}_r$
	and the terminal cost function ${V}_{f}$ are de"0²3ned as ${V}_{f}(\boldsymbol{x}) = (\boldsymbol{x}-\boldsymbol{x}_r)^{\top}\boldsymbol{P}(\boldsymbol{x}-\boldsymbol{x}_r)$, where $\boldsymbol{K} \in \mathbb{R}^3$ and $\boldsymbol{P} \in \mathbb{R}^{3\times3}$.
	The terminal cost is computed by solving the linear quadratic regulator (LQR) for the linearized model near the reference $\boldsymbol{x}_r$.
	The linearization of the nominal NARX model $\boldsymbol{x}_{k+1} =F \left( {\boldsymbol{x}}_{ k},{\boldsymbol{u}}_{k}; \boldsymbol{\Theta},\mathcal{D} \right)$ with $\boldsymbol{x}_k = [y_k, y_{k-1}, y_{k-2}]$ takes the form $\boldsymbol{x}_k=\boldsymbol{A}\boldsymbol{x}_k+\boldsymbol{B}{u}_k$,
	\begin{align}
		\left[ \begin{matrix}
			{{y}_{k+1}}  \\
			{{y}_{k}}  \\
			{{y}_{k-1}}  \\
		\end{matrix} \right]=\underbrace{\left[ \begin{matrix}
				{{a}_{11}} & {{a}_{12}} & {{a}_{13}}  \\
				1 & 0 & 0  \\
				0 & 1 & 0  \\
			\end{matrix} \right]}_{\boldsymbol{A}}\left[ \begin{matrix}
			{{y}_{k}}  \\
			{{y}_{k-1}}  \\
			{{y}_{k-2}}  \\
		\end{matrix} \right]+\underbrace{\left[ \begin{matrix}
				{{b}_{1}}  \\
				0  \\
				0  \\
			\end{matrix} \right]}_{\boldsymbol{B}}{{u}_{k}}
	\end{align}
	where $a_{11}=\frac{\partial{\mathbb{E}[y_{k}]}}{\partial{y_{k+1}}}$, $a_{12}=\frac{\partial{\mathbb{E}[y_{k}]}}{\partial{y_{k-1}}}$, $a_{13}=\frac{\partial{\mathbb{E}[y_{k}]}}{\partial{y_{k-2}}}$ and
	$b_1=\frac{\partial{\mathbb{E}[y_{k}]}}{\partial{u_{k}}}$ are evaluated at $\boldsymbol{u}_r,\boldsymbol{y}_r$ and $\boldsymbol{u}_r+\boldsymbol{\bigtriangleup},\boldsymbol{y}_r+\boldsymbol{\bigtriangleup}$. $\boldsymbol{\bigtriangleup}$ is a small perturbation and $\boldsymbol{\bigtriangleup} = 0.015$.
	The coef"0²3cients for the linearized system evaluated at the reference point, $a_{11}=3.124,a_{12}=-2.124,a_{13}=-0.078$ and $b_1=0.002$. 

	A barrier function was added to the stage cost to account for soft output constraints
	\begin{align} \label{62}
		\begin{split}
		\ell(\boldsymbol{x}_k,{u}_k)&=||\boldsymbol{x}_k-\boldsymbol{x} _r||^2_{\boldsymbol{Q}} + || u_k-u_r ||^2_{\boldsymbol{R}} \\
			&+\eta\left( 1-\exp(\frac{-\delta(y_k-\mathcal{Y})}{\rho})     \right) 
		\end{split}
	\end{align}
	where $Q=50$, $R=1$, $\eta=100$, and $\rho=1$.
	Using the LQR method, $\boldsymbol{P}$ and $\boldsymbol{K}$ can be obtained
	\begin{align}
		\begin{split}
			\boldsymbol{p}&=\left[ \begin{matrix}
				1204&	-2532&	-6518 \\
				-2532&	9278&	11739 \\
				-6518&	11739&	42656  \\
			\end{matrix} \right]  \\
			\boldsymbol{K} &= \left[ 3.384	-6.653	-0.077\right] 
		\end{split}
	\end{align}

	Consider the selection of terminal costs and terminal controllers, and also the selection of the total stage cost of \eqref{62}. Then given the prediction error in closed-loop operation and the NSVB method. Assumptions \ref{ass1} - \ref{ass-4-8a} are satisfied and the closed-loop system is guaranteed to be ISS by Theorem \ref{th4}.

The tests for this work were performed on my laptop with a cpu of AMD Ryzen 7 5800H at 3.20 GHz and 16GB of RAM at 3200MHz. The system is Windows 11 23H2 and MATLAB version 2022b.

In order to show the superiority of the developed NSVB method, the GP method was chosen as a comparison.
Both methods experiment with the collected data to learn the model. 
We tested the model learning ability of both methods with limited training data. We use the first 250 data points to learn the model, and then use the subsequent 750 data points to verify the reliability of model learning.
The test results of GP are shown in Figure \ref{fig.2}.
The GP method demonstrates excellent model learning performance for the first 300 data. 
But after 300 points, GP deviates from the true value in the output and the prediction error of $\boldsymbol{x}_1$ is close to 50 and the prediction error of $\boldsymbol{x}_3$ is close to 10. 
The test results of NSVB are shown in Figure \ref{fig.3}. $\boldsymbol{x}_3$ shows a very high accuracy with prediction error close to 2.5, and the prediction error of $\boldsymbol{x}_1$ is close to 25. Thus, it can be seen that the NSVB method has a better performance.
Existing machine learning and deep learning models generally fail to achieve good out-of-sample performance with very small training samples. Models trained with small samples can easily fall into overfitting to small samples and underfitting to the target task. Sparse models can automatically select important features, which can avoid overfitting and thus improve generalization ability.
The NSVB method leverages the sparsity of the data by selecting only the features or parameters that significantly contribute to the prediction results. 
As a result, NSVB reduces unnecessary computational and storage overheads while improving the generalization ability and interpretability of the model.

\begin{figure}
	\centering
	\includegraphics[width=\columnwidth]{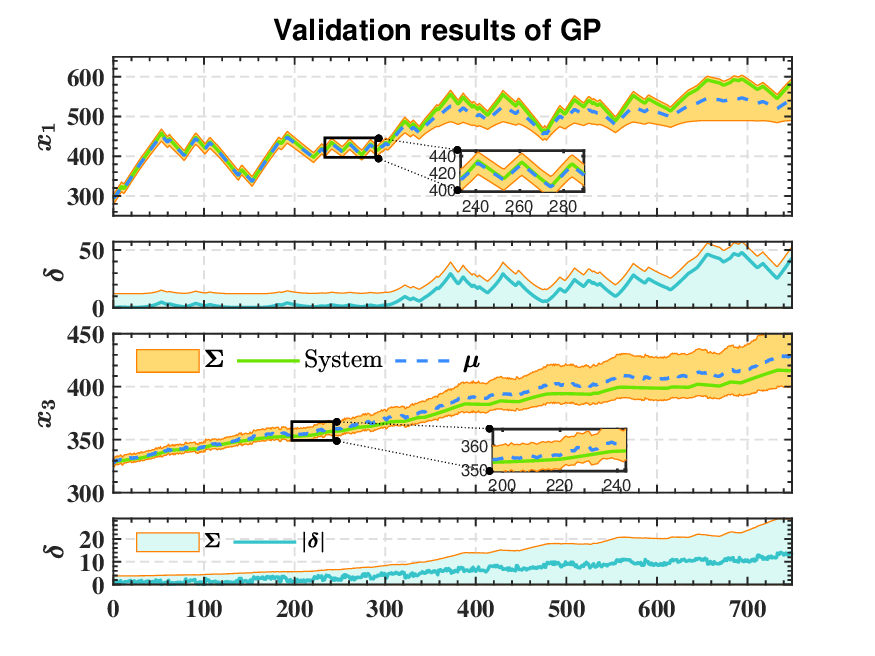}
	\caption{ GP validation results with limited training data.}
	\label{fig.2}
\end{figure}

\begin{figure}
	\centering
	\includegraphics[width=\columnwidth]{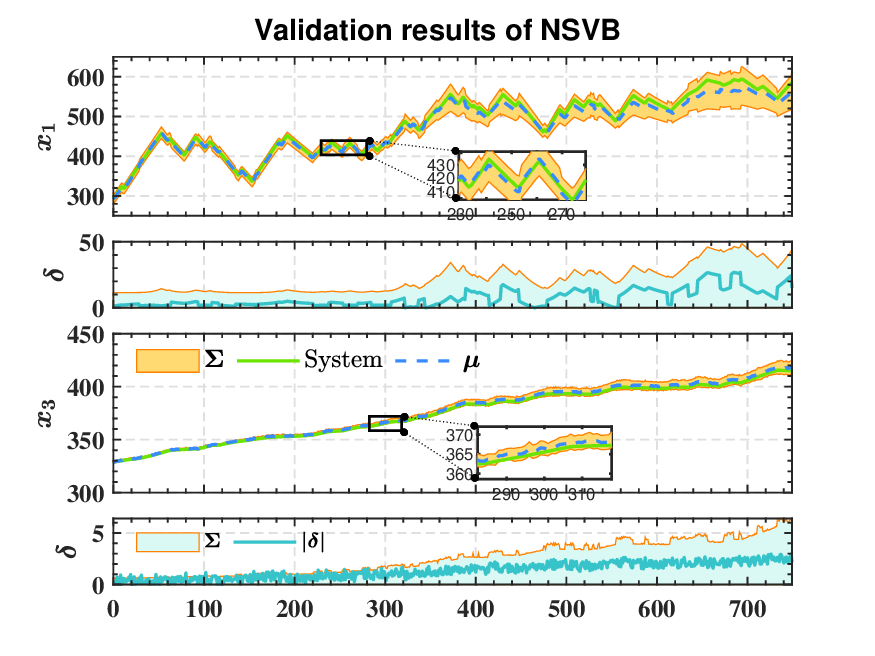}
	\caption{ NSVB validation results with limited training data.}
	\label{fig.3}
\end{figure}


In addition, we tested the model learning performance of the two methods under noise. The training sample is the first 500 data points added with triple random measurement noise, and the validation sample is the last 750 data points. 
The test results of GP are shown in Figure \ref{fig.4}. $\boldsymbol{x}_1$ gradually deviates from the true value with a prediction error close to 180. $\boldsymbol{x}_3$ has a prediction error close to 20.  The test results of NSVB are shown in Figure \ref{fig.5}. To our surprise, the test result of $\boldsymbol{x}_1$ is excellent and maintains a very high fit under noise. 
The prediction error is close to 13 for $\boldsymbol{x}_1$ and 10 for $\boldsymbol{x}_1$. This further validates the effectiveness of the NSVB method as sparsity reduces the interference of noise.
GP does not fit the real data more accurately under noise perturbation. 
It also shows that establishing Bayesian sparse method can effectively improve the reliability and generalization of model learning.

\begin{figure}
	\centering
	\includegraphics[width=\columnwidth]{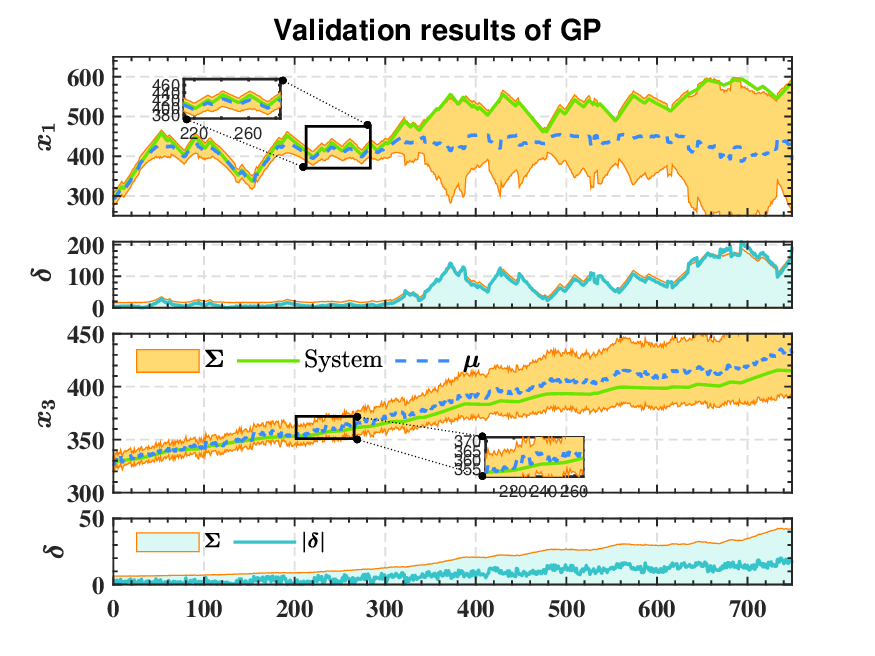}
	\caption{ GP validation results with noise.}
	\label{fig.4}
\end{figure}

\begin{figure}
	\centering
	\includegraphics[width=\columnwidth]{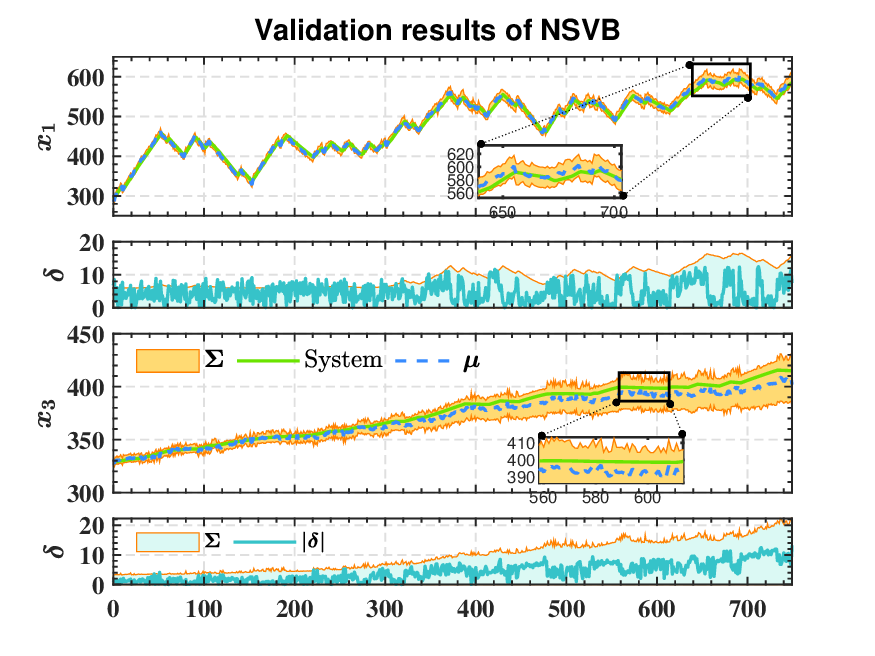}
	\caption{ NSVB validation results with noise.}
	\label{fig.5}
\end{figure}

	In the control performance test, the ODE-MPC is considered optimal because the true model is used as a predictive model.
	The test results of ODE-MPC are shown in Figure \ref{fig.6}. 
	The effect of external noise is not considered in the tests of ODE-MPC.
	The test results of ODE-MPC illustrate the validity of the ODE equations and controllers.
	\begin{figure}
		\centering
		\includegraphics[width=\columnwidth]{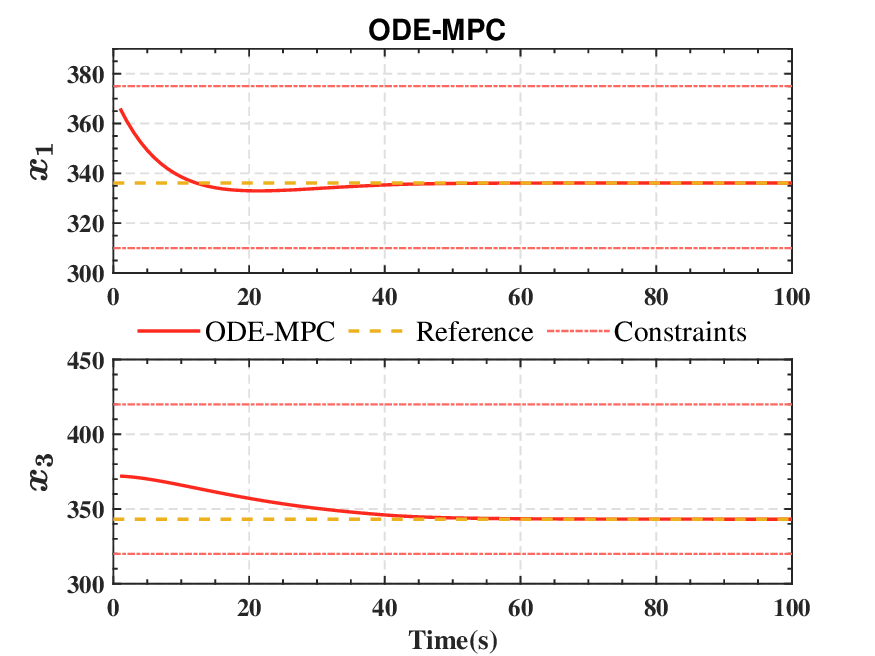}
		\caption{ ODE-MPC.}
		\label{fig.6}
	\end{figure}
	Although the NSVB method has better predictive performance, which suggests that NSVB can learn more accurate models. However, there is a need to test the control performance of NSVB.

	The test results of GP-MPC are shown in Figure \ref{fig.7}, which shows that GP-MPC  is equipped with good control performance under noise perturbation. 
	Although the prediction performance of GP is not as good as that of NSVB, the effect of the controller still ensures the tracking of the set reference.
	\begin{figure}
		\centering
		\includegraphics[width=\columnwidth]{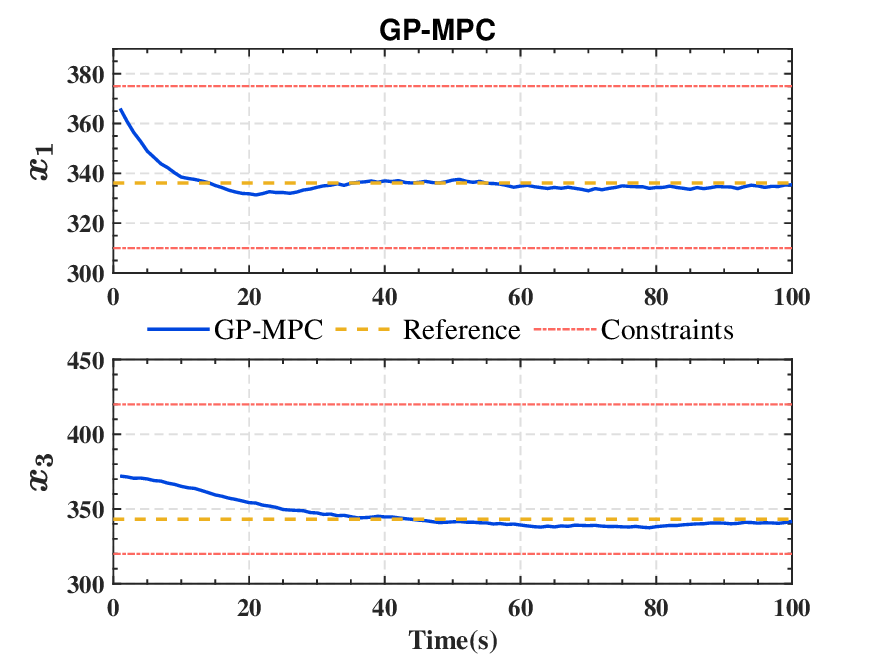}
		\caption{ GP-MPC.}
		\label{fig.7}
	\end{figure}

	The test results of NSVB is shown in Figure \ref{fig.8}. It can be seen that the NSVB-MPC method still has better control performance than GP-MPC. 
	Since the controllers of GP-MPC and NSVB-MPC are the same, the control results look somewhat similar. 
	However, considering the larger scale industrial applications, NSVB-MPC will show worthwhile control performance. The control law of NSVB-MPC is shown in Figure \ref{fig.9}.
	\begin{figure}
		\centering
		\includegraphics[width=\columnwidth]{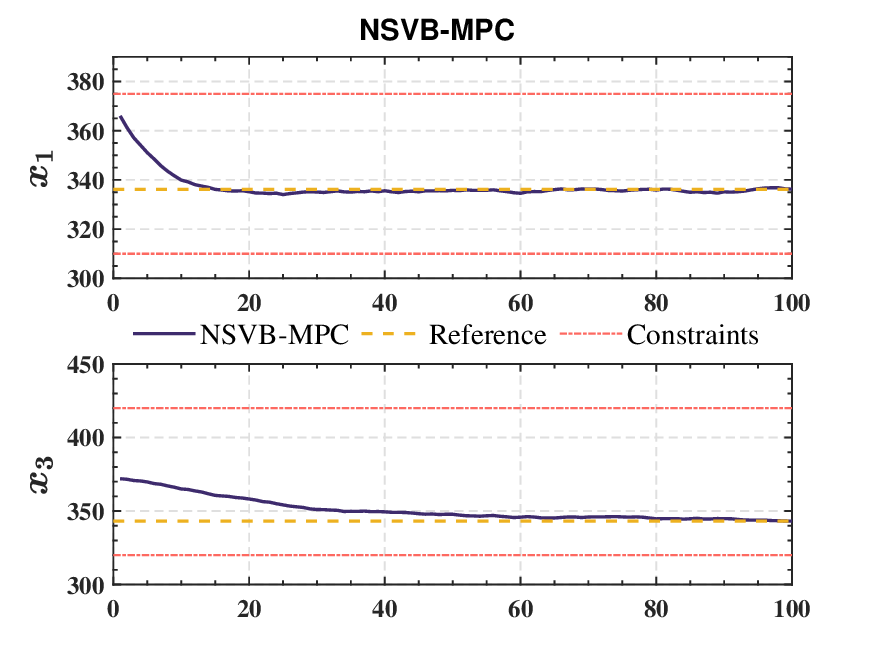}
		\caption{ NSVB-MPC.}
		\label{fig.8}
	\end{figure}
	\begin{figure}
		\centering
		\includegraphics[width=\columnwidth]{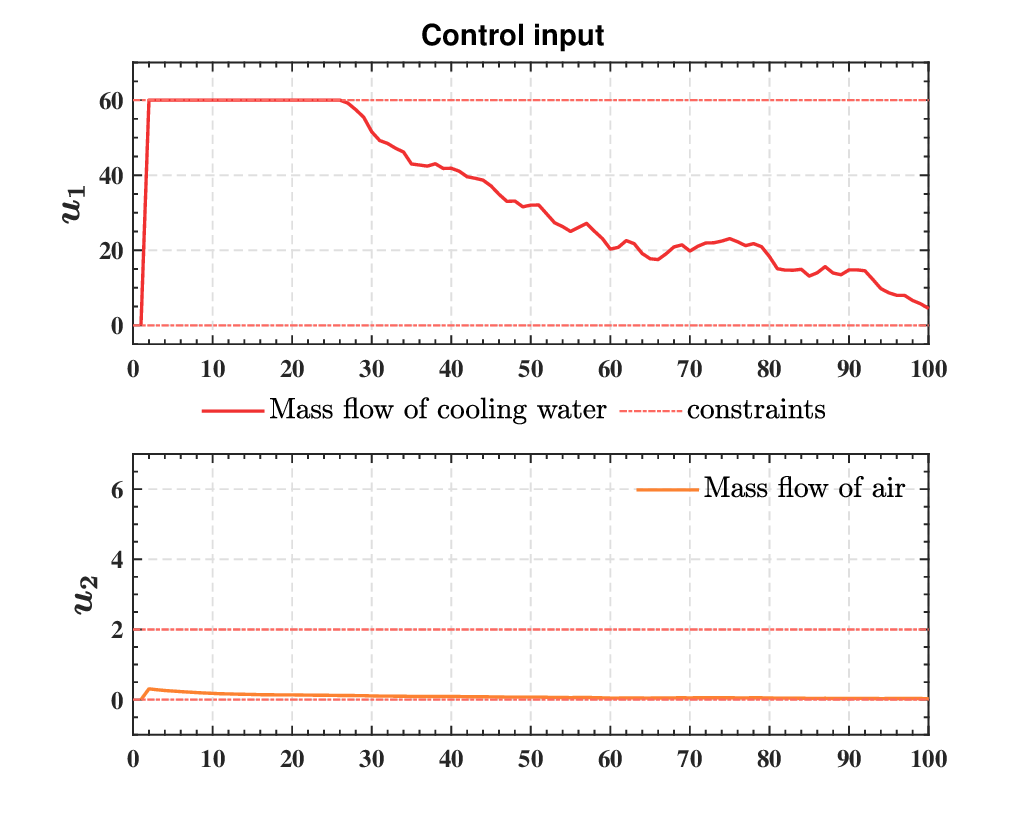}
		\caption{ Control law of NSVB-MPC.}
		\label{fig.9}
	\end{figure}

Finally, a comparison of the cost functions of the two methods is shown in Figure \ref{fig.10}. It can be seen that NSVB-MPC not only converges faster, but also has a larger cost function. 
This illustrates the significant superiority of NSVB-MPC. 
The running times of the two methods are shown in Table \ref{tbl1}. It can be seen that the NSVB-MPC method takes much less time than GP-MPC. 
This suggests that sparse can significantly reduce the time complexity of the model and process the data in a more efficient manner. 
It is worth noting that the laptop used by the authors has limited performance. On industrial scenarios, NSVb-MPC has potential payoffs in terms of computational efficiency.

	\begin{figure}
	\centering
	\includegraphics[width=\columnwidth]{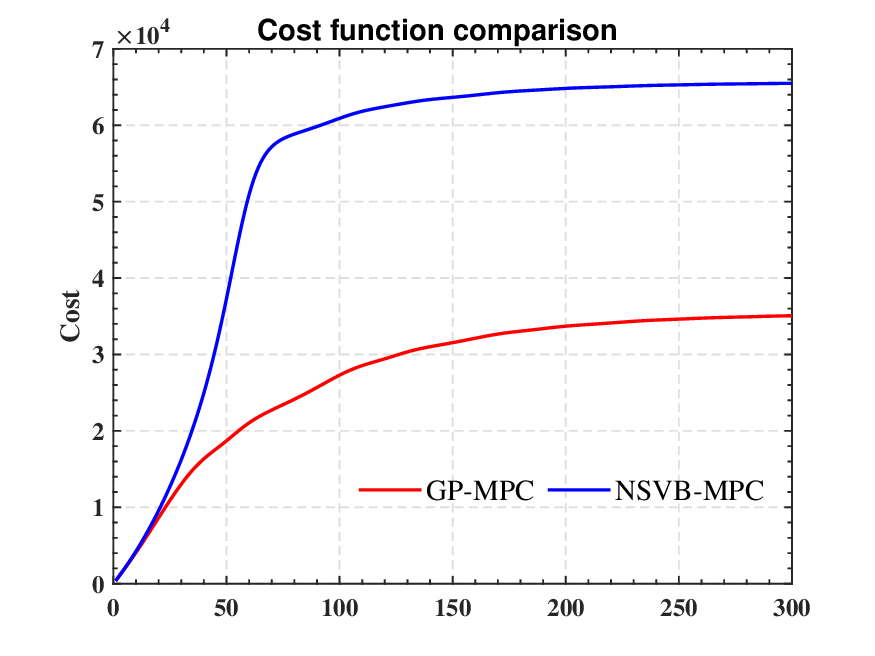}
	\caption{ Cost function.}
	\label{fig.10}
	\end{figure}

\begin{table}
	[width=.9\linewidth,cols=4,pos=h]
	\caption{Running time.}\label{tbl1}
	\begin{tabular*}{\tblwidth}{@{} LLLL@{} }
		\toprule
		  Method & GP-MPC  & NSVB-MPC\\
		\midrule
		Time (s) & 120  & 52  \\
		\bottomrule
	\end{tabular*}
\end{table}

\section{Conclusions} \label{sec:conclusions}

In this study, we introduce a data-driven MPC approach that leverages recent developments in machine learning, emphasizing the model learning process using Bayesian methods. Utilizing available data, learning-based MPC techniques enhance the precision of model predictions. Our research focuses on the development of a variational Bayesian inference framework for MPC methods, ensuring both stability and robustness. In the model learning stage, we account for noise in model analysis and quantify uncertainty, resulting in more dependable model acquisition. In the predictive control phase, our controller, devoid of terminal constraints, eliminates the need for constructing robust invariant sets while still providing stability and robustness assurances. Finally, we validate the practicality of NSVB-MPC through temperature control in a PEMFC system.

\section*{CRediT authorship contribution statement}

\textbf{Qi Zhang}: Conceptualization, Data curation, Formal analysis, Investigation, Methodology, Resources, Software, Validation, Visualization, Writing - original draft, Writing - review \& editing. \textbf{Lei Wang}: Conceptualization, Supervision. \textbf{Weihua Xu}: Supervision.  \textbf{Hongye Su}: Project administration. \textbf{Lei Xie}: Funding acquisition, Supervision.

\section*{Declaration of competing interest}
The authors declare that they have no known competing financial interests or personal relationships that could have appeared to influence the work reported in this paper.

\section*{Acknowledgments} 
We would like to acknowledge the time and effort devoted by anonymous reviewers to improving the quality of this paper. 

\bibliographystyle{cas-model2-names}

\bibliography{MPC}


%
%

\end{document}